\documentclass{article}

\usepackage{microtype}
\usepackage{graphicx}
\usepackage{subcaption}
\usepackage{booktabs} 
\usepackage{multirow}
\usepackage{hyperref}

\usepackage{xcolor}
\usepackage{amsmath}

\usepackage[preprint]{icml2026}

\usepackage{amsmath}
\usepackage{amssymb}
\usepackage{mathtools}
\usepackage{amsthm}
\usepackage{dsfont}

\newcommand{\R}{\mathbb{R}}
\newcommand{\Prob}{\mathbb{P}}
\newcommand{\E}{\mathbb{E}}
\newcommand{\Var}{\text{Var}}
\newcommand{\calC}{\mathcal{C}}

\usepackage[capitalize,noabbrev]{cleveref}

\theoremstyle{plain}
\newtheorem{theorem}{Theorem}[section]
\newtheorem{proposition}[theorem]{Proposition}
\newtheorem{lemma}[theorem]{Lemma}

\theoremstyle{definition}

\newtheorem{assumption}[theorem]{Assumption}
\theoremstyle{remark}
\newtheorem{remark}[theorem]{Remark}

\usepackage{stmaryrd}
\usepackage{xstring}

\newcommand{\multiset}[2][a]{
\IfEqCase{#1}{%
{a}{\Lbag#2\Rbag}%
{0}{\Lbag#2\Rbag}%
{1}{\big\Lbag#2\big\Rbag}%
{2}{\Big\Lbag#2\Big\Rbag}%
{3}{\bigg\Lbag#2\bigg\Rbag}%
{4}{\Bigg\Lbag#2\Bigg\Rbag}%
}[\PackageError{multiset}{Undefined option to multiset: #1}{}]%
}

\usepackage[textsize=tiny]{todonotes}

\icmltitlerunning{Distribution-informed Efficient Conformal Prediction for Full Ranking}

\begin{document}

\twocolumn[
  \icmltitle{Distribution-informed Efficient Conformal Prediction for Full Ranking}



  \icmlsetsymbol{equal}{*}
  \icmlsetsymbol{cor}{$\dagger$}

  \begin{icmlauthorlist}
    \icmlauthor{Wenbo Liao}{equal,xxx,yyy}
    \icmlauthor{Huipeng Huang}{equal,yyy}
    \icmlauthor{Chen Jia}{zzz}
    \icmlauthor{Huajun Xi}{yyy}
    \icmlauthor{Hao Zeng}{yyy}
    \icmlauthor{Hongxin Wei}{cor,yyy}

  \end{icmlauthorlist}

  \icmlaffiliation{xxx}{Department of Mathematics, Chinese University of Hong Kong}
  \icmlaffiliation{yyy}{Department of Statistics and Data Science, Southern University of Science and Technology}
  \icmlaffiliation{zzz}{School of Computation, Information, and Technology, Technical University of Munich}

  \icmlcorrespondingauthor{Hongxin Wei}{weihx@sustech.edu.cn}

  \icmlkeywords{Machine Learning, ICML}

  \vskip 0.3in
]



\printAffiliationsAndNotice{}  

\begin{abstract}
Quantifying uncertainty is critical for the safe deployment of ranking models in real-world applications.
Recent work offers a rigorous solution using conformal prediction in a full ranking scenario, which aims to construct prediction sets for the absolute ranks of test items based on the relative ranks of calibration items.
However, relying on upper bounds of non-conformity scores renders the method overly conservative, resulting in substantially large prediction sets.
To address this, we propose Distribution-informed Conformal Ranking (\textbf{DCR}), which produces efficient prediction sets by deriving the exact distribution of non-conformity scores.
In particular, we find that the absolute ranks of calibration items follow Negative Hypergeometric distributions, conditional on their relative ranks.
DCR thus uses the rank distribution to derive non-conformity score distribution and determine conformal thresholds.
We provide theoretical guarantees that DCR achieves improved efficiency over the baseline while ensuring valid coverage under mild assumptions.
Extensive experiments demonstrate the superiority of DCR, reducing average prediction set size by up to $36\%$, while maintaining valid coverage.
\end{abstract}

\section{Introduction}
Ranking models are widely deployed in real-world applications, including search engines \citep{xiong2024search}, recommendation systems \citep{karatzoglou2013learning, hou2024large}, and information retrieval \citep{cao2006adapting, liu2009learning}.
However, most ranking models only output deterministic orderings without conveying the uncertainty associated with their predicted ranks, which hinders their safe and trustworthy deployment.
The challenge highlights the importance of uncertainty quantification for ranking models, which enables deployed systems to recognize when their predicted ranks are unreliable.



To this end, Transductive Conformal Prediction for Ranking (TCPR) \citep{fermanian2025transductive} proposes to construct prediction sets for the ranks of test items in a full ranking scenario.
Formally, given $n + m$ items, we observe the relative ranks of $n$ calibration items and seek to construct prediction sets for the absolute ranks of $m$ test items with a pre-trained ranking model.
Classical conformal prediction (CP) methods cannot be directly applied here because the non-conformity scores of calibration items cannot be computed without knowing their absolute ranks, which in turn depend on the unknown ranks of test items.
TCPR addresses this by constructing high-probability upper bounds for non-conformity scores. 
However, this upper-bound approach is conservative, yielding excessively large prediction sets that do not faithfully articulate the uncertainty of the ranking model.
This limitation motivates us to design a method for the full ranking scenario that bypasses conservative upper bounds and produces significantly efficient prediction sets.

In this work, we introduce Distribution-informed Conformal Ranking (\textbf{DCR}), a novel method that constructs efficient prediction sets for absolute ranks of test items with valid coverage.
In particular, we establish that the absolute ranks of calibration items follow Negative Hypergeometric distributions, conditional on their relative ranks.
This allows us to derive the exact distribution of unobservable non-conformity scores of calibration items in closed form, without resorting to conservative upper bounds.
The precise thresholds derived from the score distribution are then incorporated into the CP framework to produce efficient prediction sets for test items.
We provide rigorous theoretical guarantees that DCR achieves valid coverage under the exchangeability assumption. 
We further prove that DCR is guaranteed to produce smaller prediction sets than TCPR at the same miscoverage rate.
Besides, to further address scalability in large-scale ranking problems, we also propose a stochastic variant of DCR, termed Monte-Carlo Distribution-informed Conformal Ranking (MDCR), which significantly reduces computational complexity compared to DCR while preserving marginal coverage guarantees.
Our method imposes no parametric assumptions about the data distribution and is applicable to any black-box ranking model, making it a practical tool for quantifying ranking uncertainty.

To verify the effectiveness of our method, we conduct extensive evaluations on various models and datasets.
The results demonstrate that DCR produces significantly more efficient prediction sets with valid coverage.
For example, on ESOL \citep{delaney2004esol} with LambdaMart \citep{liu2009learning} at a target coverage rate $90\%$, DCR produces a relative length of 57.64\%, while maintaining coverage above $90\%$.
In comparison, TCPR produces a relative length of 72.50\%.
Besides, we conduct additional analyzes on how effectively DCR can exploit the rank distribution.
Through additional analyzes, we show that as the calibration set size increases or the test set size decreases, DCR achieves realized coverage increasingly close to the target level by more effectively exploiting the rank distribution.
Additionally, we find that the VA score is more adaptive than the RA score.

We summarize our contributions as follows:
    \begin{itemize}
        \item We propose DCR, a novel method that constructs efficient prediction sets for absolute ranks of test items with valid coverage.
        The core idea behind DCR is to derive the non-conformity score distribution based on the absolute rank distribution.
        \item We theoretically prove that DCR achieves valid coverage under the exchangeability assumption and is guaranteed to be more efficient than the baseline in expectation at the same miscoverage rate.
        \item We perform additional analyzes that lead to improved understandings of our method.
        In particular, we find that as the calibration set size increases, DCR achieves realized coverage increasingly close to the target level by more effectively exploiting the rank distribution.
    \end{itemize}

\section{Background}

\paragraph{Full ranking setup.}
\label{sec:setting}

We study the problem of quantifying ranking uncertainty in a full ranking scenario \citep{liu2009learning, rising2021uncertainty, chen2022optimal}.
Let $\mathcal{X}$ denote the input space and $\mathbb{R}$ denote the score space, which helps us properly define each item's rank.
We have a collection of $n+m$ items, each associated with a pair $(X_i, Y_i) \in \mathcal{X} \times \mathbb{R}$.
These items are partitioned into a calibration set of size $n$, $\{(X_i, Y_i)\}_{i \in [n]}$ and a test set of size $m$, $\{(X_{n+j}, Y_{n+j})\}_{j \in [m]}$, where $[n] := \{1,\ldots, n\}$.
To formalize ranking relationships, for any $Y \in \mathbb{R}$ and any finite subset $\mathcal{D} \subset \mathbb{R}$, we define the rank of $Y$ in $\mathcal{D}$ as
$
\mathrm{R}(Y,\mathcal{D}) := \sum_{z \in \mathcal{D}} \mathbf{1}\{Y \ge z\}.
$
Using this definition, we define the position of item $i$
within the calibration set by
$
R_i^c = \mathrm{R}\!\left(Y_i, \Lbag Y_j \Rbag_{j\in[n]}\right),
$
within the test set by
$
R_i^t = \mathrm{R}\!\left(Y_i, \Lbag Y_{n+j} \Rbag_{j\in[m]}\right),
$
and within the entire dataset by
$
R_i^{c+t} = R_i^c + R_i^t .
$


In this work, we assume that the calibration scores $\{Y_i\}_{i\in[n]}$ are known but test scores $\{Y_{n+j}\}_{j\in[m]}$ are unobserved.
Thus, the relative ranks of calibration items $\{R_i^{c}\}_{i\in[n]}$ are known, but the absolute rank of any item is unknown.
To establish statistical guarantees, we make the following assumptions as in prior work \citep{fermanian2025transductive}.

\begin{assumption}
    \label{ass:exchange}
    The vector $(X_i, Y_i)_{i\in [n+m]}$ is exchangeable and $(Y_i)_{i \in [n+m]}$ has no ties almost surely.
\end{assumption}

\begin{remark}[Handling Ties]
\label{rem:ties}
To strictly satisfy the no-tie condition in Assumption~\ref{ass:exchange}, we apply a standard random tie-breaking procedure in our experiments. Specifically, we add microscopic noise $\xi \sim \mathcal{U}(0, \epsilon)$ to the scores, which ensures a unique ranking order almost surely without affecting the substantive performance metrics.
\end{remark}

We assume access to a pre-trained ranking model
$
A : \mathcal{X} \times \mathcal{X}^{n+m} \to \mathbb{K},
$
which predicts the rank of a point inside a multiset.
Its inputs are the target item $X_i$ and the items $\Lbag X_j \Rbag_{j\in [n+m]}$ among which it seeks to sort.
We use $\Lbag\cdot\Rbag$ to denote a multiset (or bag), which is an unordered set with potentially repeated elements.
For clarity, when there is no confusion, we will omit the dependence in the multiset: $A(X_i):=A(X_i, \Lbag X_j \Rbag_{j\in [n+m]})$.
$A(X_i)$ could either be a rank or a value from which a rank can be deduced.

However, the predicted output $A(X_i)$ unavoidably contains errors.
To quantify errors, we define two uncertainty non-conformity score functions based on the output type of $A$.

\textbf{(RA) setting:} $\mathbb{K} = \mathbb{N}$, the rank is directly predicted by $A$: 
$\widehat{R}_i^{c+t}=A(X_i)$. 
In this case, we consider the residual score:
\begin{align*}
    s^{RA}(X_i, r) := |r - \widehat{R}_i^{c+t}|
\end{align*}
This is the difference between the rank $r$ and the predicted absolute rank of $X_i$ among $n+m$ items.

\textbf{(VA) setting:} $\mathbb{K} = \mathbb{R}$, the predicted rank is deduced from values output by $A$:
$\widehat{R}_{i}^{c+t}=\mathrm{R}(A(X_i), \Lbag A(X_j)\Rbag_{j \in [n+m]}).$ 
As $A(X_i)$ is a real value, we consider a refined score function:
\begin{align*}
    s^{VA}(X_i, r) := |A(X_i) - \mathrm{R}^{-1}(r, \Lbag A(X_j) \Rbag_{j \in [n+m]})|, 
\end{align*}
where $\mathrm{R}^{-1}(r,\mathcal{D}) := \sum_{z\in \mathcal{D}}z\mathbf{1}\{\mathrm{R}(z,\mathcal{D}) = r\}$ is the $r$-th smallest value in $\mathcal{D}$ if $\mathcal{D}$ has no ties.
This score quantifies the distance between the predicted value of $X_i$ and the score it should have had to be at rank $r$.

\paragraph{Transductive conformal prediction for ranking (TCPR).}
CP constructs prediction sets with finite-sample coverage guarantees by computing non-conformity scores on a calibration set and determining thresholds from the empirical distribution of the calibration scores.
We provide a detailed introduction to CP in Appendix \ref{sec: cp_intro}.
However, classical CP methods cannot be applied in the full ranking setting because the non-conformity scores of calibration items $s(X_i, R_i^{c+t})$ cannot be computed without knowing their absolute ranks $R_i^{c+t} = R_i^c + R_i^t$, where the terms $R_i^t$ depend on the unobserved test scores and are thus unknown.

TCPR addresses this challenge by constructing distribution-free high-probability bounds for the unknown absolute ranks of calibration items.
Specifically, TCPR constructs random variables $R_i^- , R_i^+ \in \mathbb{R}$ for $i\in[n]$,
which depend only on the observed data $(X_i, R_i^c)_{i \in [n]}$ and satisfy
\begin{equation}
\label{eq:tcpr-rank-bounds}
\mathbb{P}\!\left(
\forall i\in[n]: R_i^{c+t} \in [R_i^-, R_i^+]
\right) \ge 1-\delta,
\end{equation}
where $\delta \in (0,1)$ is a confidence level satisfying $\delta < \alpha$, and $\alpha \in (0,1)$ is the target miscoverage rate. 
TCPR proposes three types of absolute rank bounds: theoretical bounds, linear envelopes, and quantile envelopes. 
We provide a detailed introduction to these rank bounds in Appendix~\ref{sec:rank_bound}.

After obtaining bounds for $\{R_i^{c+t}\}_{i \in [n]}$, TCPR constructs high-probability upper bounds for the non-conformity scores of calibration items:
\begin{equation}
\label{eq:tcpr-proxy-score}
\tilde{s}_i
\;:=\;
\max_{r \in [R_i^-,\,R_i^+]}
s(X_i, r),
\qquad i\in[n].
\end{equation}
With probability at least $1-\delta$, the non-conformity scores of calibration items satisfy
$s_i \le \tilde{s}_i$ for all $i \in [n]$.
The threshold $s^*_{TCPR}$ is then determined using the $(1-\alpha+\delta)$-th empirical quantile of the upper bounds:
\begin{equation}\label{eq:threshold}
s^*_{TCPR} = \text{Quantile}(\frac{\lceil(n+1)(1-\alpha+\delta)\rceil}{n}, \{\tilde{s}_i\}_{i\in[n]}).
\end{equation}
For a test item $X_{n+j}$, TCPR constructs a prediction set for its absolute rank by
\begin{equation}
\label{eq:tcpr-pred-set}
\mathcal{C}(X_{n+j})
=
\bigl\{
r \in [n+m] :
s(X_{n+j}, r) \le s^*_{TCPR}
\bigr\}.
\end{equation}

\begin{theorem}
\label{thm:tcpr-validity}
If Assumption \ref{ass:exchange} holds and the rank bounds
$\{(R_i^-,R_i^+)\}_{i\in[n]}$ satisfy \eqref{eq:tcpr-rank-bounds},
then for any test item $X_{n+j}$ , the prediction set defined in \eqref{eq:tcpr-pred-set} satisfies
\[
\mathbb{P}\!\left(
R_{n+j}^{c+t} \in \mathcal{C}(X_{n+j})
\right) \geq (1-\alpha+\delta)(1-\delta) \geq 1-\alpha.
\]
\end{theorem}

However, TCPR is excessively conservative, which stems from its reliance on constructing upper bounds for the calibration non-conformity scores.
Specifically, TCPR replaces the calibration scores $s(X_i, R_i^{c+t})$ with proxy scores $\tilde{s}_i$, which are defined as worst-case values over the rank interval $[R_i^-, R_i^+]$ and thus constitute loose overestimates of the true non-conformity scores.
Furthermore, the rank bounds $[R_i^-, R_i^+]$ must hold uniformly over all calibration items with high probability, which leads to wide rank intervals and further inflates the proxy scores.
Additionally, TCPR adjusts the target coverage level from $1-\alpha$ to $1-\alpha+\delta$ to account for uncertainty in the rank bounds, introducing extra slack.
Together, these factors lead to large thresholds $s^*_{TCPR}$, making TCPR inherently conservative.

\begin{figure}[t]
    \centering
        \includegraphics[width=0.85\linewidth]{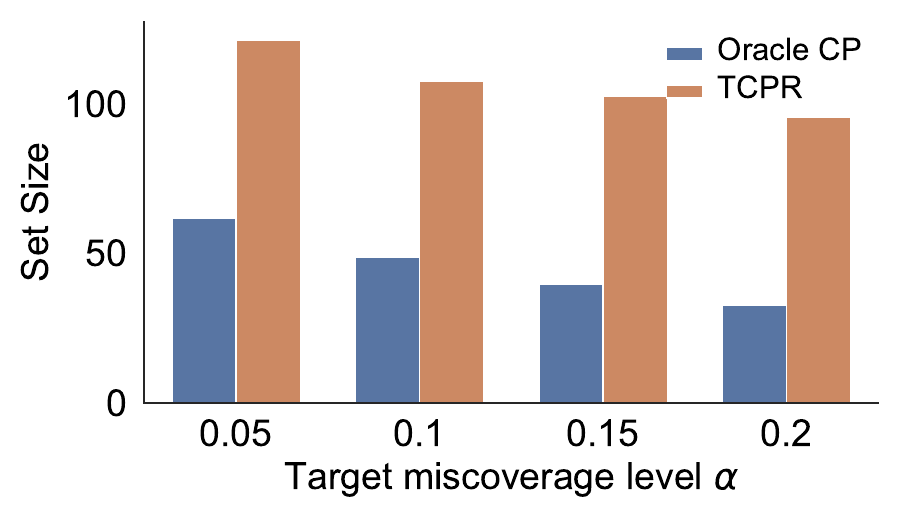}
        \caption{\textbf{Prediction set size of TCPR and oracle CP at different target error levels.}
        The experiment is conducted with RankNet \citep{10.1145/1102351.1102363} on a synthetic dataset.
        For TCPR, rank bounds are constructed using the quantile envelope, as it provides the tightest bound.
        Since both methods achieve the desired coverage, only the prediction set size is presented.}
    \label{fig:tcpr_vs_oracle}
\end{figure}

To provide a straightforward view, we compare the prediction set sizes of TCPR and oracle CP on a synthetic dataset in Figure \ref{fig:tcpr_vs_oracle}, where oracle CP computes non-conformity scores using the true absolute ranks of calibration items.
The experimental setup is detailed in Appendix \ref{app:details}.
The results show that TCPR produces substantially larger prediction sets than oracle CP, highlighting the conservatism of TCPR.
We proceed by introducing our method, which eliminates the need for upper-bound constructions and achieves superior efficiency.



\section{Methodology and theoretical analysis}
\label{sec:method_and_theory}
In this section, we introduce Distribution-informed Conformal Ranking (\textbf{DCR}), a rigorous methodology for quantifying uncertainty in full ranking problems. In contrast to prior approaches that rely on conservative high-probability envelopes to bound unobserved ranks, DCR exploits the exact probabilistic structure of the latent non-conformity scores. This distributional perspective allows us to bypass loose approximations, facilitating the construction of significantly tighter prediction sets while rigorously maintaining the validity of the coverage guarantee.

\paragraph{Probabilistic modeling of unobserved ranks.} 
The core challenge in the transductive ranking setting is that the true non-conformity scores $S_i = s(X_i, R_i^{c+t})$ for the calibration set are effectively latent variables. This is because the absolute rank $R_i^{c+t}$ depends on the relative ordering of item $i$ against the unseen test instances, making the scores unobservable prior to the final prediction. However, the randomness of these ranks is structured rather than arbitrary. Crucially, under the exchangeability assumption (Assumption~\ref{ass:exchange}), we establish that the unobserved portion of the rank follows a known, tractable distribution conditioned on the observed relative rank. We formally establish this relationship in the following proposition.
\begin{proposition}[Conditional Rank Distribution]
\label{prop:nh}
Under Assumption~\ref{ass:exchange}, let $R_i^c$ be the relative rank of a calibration item $i \in [n]$. The random variable $R_i^t$, representing the number of test items ranked lower than item $i$, follows a Negative Hypergeometric distribution $R_i^t \mid R_i^c \sim \text{NegHypergeometric}(N, m, R_i^c)$ with the following probability mass function:
\begin{equation}
    \mathbb{P}(R_i^t = k \mid R_i^c = r) = \frac{\binom{r+k-1}{k} \binom{N-r-k}{m-k}}{\binom{N}{m}}, \quad k \in [m],
\end{equation}
where $N=n+m$ is the total number of items. Consequently, the absolute rank $R_i^{c+t} = R_i^c + R_i^t$ follows a shifted Negative Hypergeometric distribution supported on $\{R_i^c, \dots, R_i^c + m\}$.
\end{proposition}

The detailed proof of Proposition~\ref{prop:nh} is provided in Appendix \ref{app:proof_nh}. With the distributional characterization of the absolute ranks established, we can now derive the induced distribution of the non-conformity scores. 

Given the predicted absolute rank $\hat{R}_i^{c+t} = A(X_i)$ from the ranking model, the non-conformity score $S_i = s(X_i, R_i^{c+t})$ is a function of the random absolute rank $R_i^{c+t}$. Since $R_i^{c+t} \mid R_i^c$ follows a shifted Negative Hypergeometric distribution, the score $S_i$ follows the distribution induced by this underlying ranking process. Specifically, in the RA setting, $S_i \mid R_i^c \sim |R_i^{c+t} - \hat{R}_i^{c+t}|$, while in the VA setting, $S_i \mid R_i^c \sim |\mathrm{R}^{-1}(R_i^{c+t}) - \mathrm{R}^{-1}(\hat{R}_i^{c+t})|$, where $R_i^{c+t} \sim R_i^c + \text{NegHypergeometric}(N, m, R_i^c)$.

\paragraph{Deriving the latent score distribution.}
To carry out the standard conformal prediction procedure and identify the critical threshold, we ideally aim to compute the empirical cumulative distribution function (CDF) of the calibration non-conformity scores:
\begin{equation}
\label{eq:f_cal}
\hat{F}_{\text{cal}}(t) := \frac{1}{n} \sum_{i=1}^n \mathbb{I}\{S_i \le t\}.
\end{equation}
However, unlike standard conformal prediction, $\hat{F}_{\text{cal}}(t)$ is not a deterministic function observable from the data. Since the scores $S_i$ depend on the unobserved ranks $R_i^{c+t}$, $\hat{F}_{\text{cal}}(t)$ is itself a random process.

For each calibration item $i$, let $F_i(t) := \mathbb{P}(S_i \le t \mid R_i^c)$ denote the conditional CDF of its non-conformity score given the observed relative rank $R_i^c$ (which can be computed exactly using the distribution from Proposition~\ref{prop:nh}). To address this, we construct the target distribution by considering the expectation of this latent empirical distribution, conditioned on the observable relative ranks $\{R_i^c\}_{i=1}^n$. We define the \textit{Mixture CDF}, $F_{\text{mix}}(t)$, as follows:
\begin{equation}
\label{eq:f_mix_deriv}
\begin{aligned}
    F_{\text{mix}}(t) &:= \mathbb{E}\left[ \hat{F}_{\text{cal}}(t) \mid \{R_i^c\}_{i=1}^n \right] \\
    &= \mathbb{E}\left[ \frac{1}{n} \sum_{i=1}^n \mathbb{I}\{S_i \le t\} \mid \{R_i^c\}_{i=1}^n \right] \\
    &= \frac{1}{n} \sum_{i=1}^n \mathbb{P}(S_i \le t \mid R_i^c) \\
    &= \frac{1}{n} \sum_{i=1}^n F_i(t)
\end{aligned}
\end{equation}
This derivation reveals that the mixture distribution used in DCR is theoretically grounded: it represents the mean distribution of the non-conformity scores given the observed relative ranks. By using $F_{\text{mix}}$ to determine the threshold, we effectively marginalize out the aleatoric uncertainty of the unobserved absolute ranks.

\paragraph{Distribution-informed conformal ranking (DCR).}
We select the threshold $s^*$ as the smallest value such that the mixture probability mass exceeds the target coverage level:
\begin{equation}
\label{eq:threshold_DCR}
    s^* = \inf \left\{ t \in \mathbb{R} : F_{\text{mix}}(t) \ge \frac{\lceil(n+1)(1-\alpha)\rceil}{n+1} \right\}.
\end{equation}
Using this threshold, the prediction set for a test item $X_{n+j}$ is constructed as $\mathcal{C}(X_{n+j}) = \{ r \in [N] : s(X_{n+j},r) \le s^* \}$. The procedure is summarized in Algorithm~\ref{alg:DCR}.

\begin{algorithm}[h]
   \caption{Distribution-informed Conformal Ranking (DCR)}
   \label{alg:DCR}
\begin{algorithmic}[1]
   \STATE {\bfseries Input:} Calibration data $\{(X_i, R_i^c)\}_{i=1}^n$, Test data $\{X_{n+j}\}_{j=1}^m$, Ranking Model $A$, Error rate $\alpha$.
   \STATE \textbf{Step 1: Compute Calibration Distributions}
   \FOR{$i=1$ {\bfseries to} $n$}
       \STATE Get predicted rank $\hat{R}_i^{c+t} \leftarrow A(X_i)$.
       \STATE Construct distribution of true rank $R_i^{c+t}$ using Prop.~\ref{prop:nh}.
       \STATE Compute CDF $F_i(t)$ for variable $S_i = s(X_i,R_i^{c+t})$.
   \ENDFOR
   \STATE \textbf{Step 2: Compute Threshold}
   \STATE Compute Mixture CDF: $\hat{F}_{mix}(t) \leftarrow \frac{1}{n} \sum_{i=1}^n F_i(t)$.
   \STATE Find $s^* \leftarrow \inf \{ t : \hat{F}_{mix}(t) \ge \frac{\lceil(n+1)(1-\alpha)\rceil}{n+1} \}$.
   \STATE \textbf{Step 3: Construct Prediction Sets}
   \FOR{$j=1$ {\bfseries to} $m$}
       \STATE Get predicted rank $\hat{R}_{n+j}^{c+t} \leftarrow A(X_{n+j})$.
       \STATE $\hat{\mathcal{C}}_{n+j} \leftarrow \{ r \in [N] : s(X_{n+j},r) \le s^* \}$.
   \ENDFOR
   \STATE {\bfseries Return} $\{\hat{\mathcal{C}}_{n+j}\}_{j=1}^m$
\end{algorithmic}
\end{algorithm}

We establish the marginal validity of DCR in the following theorem, the proof of which is provided in Appendix \ref{app:proof_validity}.
\begin{theorem}[Marginal Validity]
\label{thm:validity}
If Assumption \ref{ass:exchange} holds, for any target error rate $\alpha \in (0, 1)$, the prediction set $\hat{\mathcal{C}}_{n+j}$ constructed by DCR (Algorithm~\ref{alg:DCR}) satisfies:
\begin{equation}
    \mathbb{P}(R_{n+j}^{c+t} \in \hat{\mathcal{C}}_{n+j}) \ge 1 - \alpha,
\end{equation}
where the probability is taken over the randomness of both the calibration and test data.
\end{theorem}

\paragraph{Efficiency and comparison with TCPR.}
DCR addresses the conservatism in TCPR~\citep{fermanian2025transductive} by utilizing the exact distribution of non-conformity scores. TCPR relies on a high-probability envelope $[R_i^-, R_i^+]$ and uses the maximum possible score within this interval, which leads to unnecessarily large prediction sets. In contrast, DCR integrates over the rank distribution to determine the threshold. We formally establish the efficiency advantage of DCR over the TCPR baseline in the following theorem, with the detailed proof provided in Appendix \ref{app:proof_efficiency}.
\begin{theorem}[Efficiency Guarantee]
\label{thm:efficiency}
Let $s^*_{TCPR}$ be the threshold computed by TCPR~\citep{fermanian2025transductive} and $s^*_{DCR}$ be the threshold computed by DCR. We have:
\begin{equation}
    s^*_{DCR} \le s^*_{TCPR}.
\end{equation}
Consequently, the prediction set size satisfies $|\hat{\mathcal{C}}^{DCR}| \le |\hat{\mathcal{C}}^{TCPR}|$.
\end{theorem}

\medskip
\begin{remark}[Analysis of Conservatism vs. Oracle]\label{rem:conservatism}
While Theorem \ref{thm:efficiency} proves DCR is more efficient than TCPR, it remains conservative relative to the Oracle. This gap stems from the rank uncertainty modeled by the Negative Hypergeometric (NH) distribution. Specifically, in regimes of high model precision or when the calibration set is small relative to the test set ($n \ll m$), the mixture distribution exhibits significantly heavier tails compared to the realized distribution used by the Oracle. Consequently, DCR yields a higher conformal threshold to ensure validity without access to ground-truth ranks.
\end{remark}

\paragraph{Finite-Sample FCP concentration.}
Beyond marginal validity, we control the Finite-Sample False Coverage Proportion (FCP), defined as the realized fraction of uncovered test items: $\text{FCP} := \frac{1}{m} \sum_{j=1}^m \mathbb{I}\{ R_{n+j}^{c+t} \notin \hat{\mathcal{C}}_{n+j} \}$. While marginal validity ensures coverage on average, in high-stakes ranking applications, it is crucial to bound the deviation of the realized FCP from the target level $\alpha$. In the following theorem, we show that the realized FCP is well-concentrated around its expectation for large sample sizes. The detailed proof is given in Appendix \ref{app:proof_fcp}.

\begin{theorem}[Finite-Sample FCP Concentration]
\label{thm:fcp_bound}
There exists a universal constant $C > 0$ such that for any $\beta \in (0,1)$, with probability at least $1-\beta$, the FCP on the test set is bounded by:
\begin{equation}
    \text{FCP} \le \alpha + C \left( \sqrt{\frac{\ln(4/\beta)}{n}} + \sqrt{\frac{\ln(4/\beta)}{m}} \right).
\end{equation}
\end{theorem}

\medskip
\begin{remark}[Asymptotic Validity]\label{rem:asyptotic}
The bound in Theorem~\ref{thm:fcp_bound} implies that the deviation term vanishes as $n, m \to \infty$. In the large-sample regime, the probability of the FCP exceeding the target level $\alpha$ by any non-negligible margin tends to zero. In other words, the risk of the prediction sets having coverage strictly lower than $1-\alpha$ vanishes asymptotically, providing a strong validity guarantee.
\end{remark}

\paragraph{Scalable stochastic approximation of DCR.} 

To address the computational constraints of DCR in massive test set scenarios ($m \gg n$), we propose Monte-Carlo Distribution-informed Conformal Ranking (\textbf{MDCR}), a stochastic approximation that estimates the mixture distribution via efficient sampling. Algorithmically, while DCR derives the threshold by calculating the exact expectation of score distributions across the calibration set, MDCR adopts a simulation-based approach. The detailed procedure is outlined in Algorithm~\ref{alg:mDCR}.

\begin{algorithm}[h]
   \caption{Monte-Carlo Distribution-informed Conformal Ranking (MDCR)}
   \label{alg:mDCR}
\begin{algorithmic}[1]
   \STATE {\bfseries Input:} Calibration data $\{(X_i, R_i^c)\}_{i=1}^n$, Test data $\{X_{n+j}\}_{j=1}^m$, Ranking Model $A$, Error rate $\alpha$.
   \STATE \textbf{Step 1: Sample Calibration Scores}
   \FOR{$i=1$ {\bfseries to} $n$}
       \STATE Get predicted rank $\hat{R}_i^{c+t} \leftarrow A(X_i)$.
       \STATE Sample rank $\tilde{R}_i^{c+t} \sim \text{NegHyperGeom}(N, m, R_i^c)$. \COMMENT{Prop.~\ref{prop:nh}}
       \STATE Compute score $\tilde{S}_i \leftarrow s(X_i,\tilde{R}_i^{c+t})$
   \ENDFOR
   \STATE \textbf{Step 2: Compute Threshold via Sorting}
   \STATE Sort the sampled scores: $\tilde{S}_{(1)} \le \tilde{S}_{(2)} \le \dots \le \tilde{S}_{(n)}$.
   \STATE Find index $k \leftarrow \lceil(n+1)(1-\alpha)\rceil$.
   \STATE Set threshold $s^* \leftarrow \tilde{S}_{(k)}$.
   \STATE \textbf{Step 3: Construct Prediction Sets}
   \FOR{$j=1$ {\bfseries to} $m$}
       \STATE Get predicted rank $\hat{R}_{n+j}^{c+t} \leftarrow A(X_{n+j})$.
       \STATE $\hat{\mathcal{C}}_{n+j} \leftarrow \{ r \in [N] : s(X_{n+j},r) \le s^* \}$.
   \ENDFOR
   \STATE {\bfseries Return} $\{\hat{\mathcal{C}}_{n+j}\}_{j=1}^m$
\end{algorithmic}
\end{algorithm}

Since the Negative Hypergeometric distribution (Prop.~\ref{prop:nh}) characterizes the rank uncertainty of calibration items, each random draw in MDCR effectively simulates a potential realization of the ground-truth full ranking. By constructing the threshold based on these simulated realizations, MDCR directly leverages standard conformal prediction principles to ensure valid coverage. This approach shifts the computational bottleneck from dense probability summation to a simple sorting operation, achieving $\mathcal{O}(n \log n)$ complexity. In Table~\ref{tab:complexity}, we summarize the algorithmic properties of DCR, MDCR, and TCPR.

\begin{table}[h]
    \caption{Comparison of computational cost, prediction set tightness, and variance properties of different methods. $n$ is the calibration set size, $m$ is the test set size, and $K$ is the number of simulations for TCPR (typically $K=10^5$ \citet{fermanian2025transductive}).}
    \label{tab:complexity}
    \begin{center}
    \begin{small}
    \begin{sc}
    \begin{tabular}{l c c c}
    \toprule
    Method & Complexity & Tightness & Variance \\ 
    \midrule
    TCPR & $\mathcal{O}(K \cdot n \log n)$ & Low & Low \\
    DCR (Ours) & $\mathcal{O}(nm)$ & High & Low \\
    MDCR (Ours) & $\mathcal{O}(n \log n)$ & High & High \\
    \bottomrule
    \end{tabular}
    \end{sc}
    \end{small}
    \end{center}
\end{table}

While MDCR preserves the theoretical validity guarantees of DCR, the stochastic nature of its Monte-Carlo sampling introduces additional variability into the threshold estimation process. Consequently, MDCR exhibits higher variance in the resulting prediction sets compared to the deterministic DCR method, specifically in the asymptotic sense. We formally establish the marginal validity of MDCR and the asymptotic variance inequality between the two methods in the following theorems, with the proofs provided in Appendix \ref{app:proof_validity} and Appendix \ref{app:proof_variance}, respectively.

\begin{theorem}[Marginal Validity of MDCR]
\label{thm:validity_mDCR}
For any target error rate $\alpha \in (0, 1)$, the prediction set $\hat{\mathcal{C}}_{n+j}$ constructed by MDCR (Algorithm~\ref{alg:mDCR}) satisfies:
\begin{equation}
    \mathbb{P}(R_{n+j}^{c+t} \in \hat{\mathcal{C}}_{n+j}) \ge 1 - \alpha,
\end{equation}
where the probability is taken over the randomness of both the calibration and test data.
\end{theorem}

\begin{theorem}[Asymptotic Variance Comparison]
\label{thm:variance}
Let $s_{\text{DCR}}^*$ and $s_{\text{MDCR}}^*$ be the threshold estimators produced by DCR and MDCR, respectively.
Then, as the calibration size $n \to \infty$, the asymptotic variance of DCR is strictly lower than that of MDCR:
\begin{equation}
\label{eq:var_ineq}
\text{AsyVar}(s_{\text{DCR}}^*) < \text{AsyVar}(s_{\text{MDCR}}^*).
\end{equation}
\end{theorem}

\section{Experiments}
\label{sec:experiments}

In this section, we evaluate DCR and MDCR on multiple benchmark datasets.
We first describe the experimental setup, including the models, datasets, uncertainty scores, and baselines.
The source code comparing DCR, MDCR, Oracle, and TCPR on synthetic data is available at \url{https://anonymous.4open.science/r/ICML2026-0670/}.
We then present the main results, showing that DCR satisfies the target coverage constraint while achieving superior efficiency compared with baselines.
In addition, we conduct comprehensive analyses to show that larger calibration sets and smaller test sets enable DCR to more effectively exploit distributional information, and that VA score is more adaptive than RA score.

\subsection{Experimental setup}


\paragraph{Datasets and models.}
We conduct experiments on four ranking datasets, including Yummly28K \citep{wilber2015learning}, ESOL \citep{delaney2004esol}, anime recommendation \citep{ransaka2018anime}, and a synthetic dataset \citep{gazin2024transductive}.
For each dataset, we employ three widely-used ranking models: RankNet \citep{burges2005learning}, LambdaMart \citep{liu2009learning}, and RankSVM \citep{joachims2002optimizing}.
The details of datasets and models are provided in Appendix~\ref{app:details}.


\paragraph{Baselines and evaluation metrics.}
We compare DCR and MDCR against two baselines:
(1) Oracle, which computes non-conformity scores using the unobserved true ranks of the calibration items.
This serves as the theoretical lower bound for prediction set sizes.
(2) TCPR \citep{fermanian2025transductive}.
For TCPR, we use the quantile envelope to construct rank bounds, as it provides the tightest bounds in practice.
We evaluate the ranking performance by measuring the following metrics:
(1) prediction set size;
(2) false coverage proportion (FCP), defined as the fraction of test items whose true ranks are not covered by their prediction sets;
and (3) relative length: the average prediction set size normalized by the total number of items ($N=n+m$).



\paragraph{Implementation details.}
Evaluating RA and VA score functions, we set $n=100$ and $m=500$ for verification experiments, while real-world benchmarks utilize a 4:3:3 split for training, calibration, and test sets. We employ full datasets for ESOL and Yummly28k alongside a 30,000-sample subset for Anime. Ground-truth scores for Yummly28k are defined as standardized negative $L_1$ distances to a latent preference point $x^*$ with additive Gaussian noise. For the Anime dataset, relevance scores are jittered with infinitesimal noise $\mathcal{U}(-10^{-6}, 10^{-6})$ to break ties. Unless otherwise specified, we report the mean FCP and relative length averaged over 100 independent trials.



\subsection{Main Results}
\label{subsec:main_results}


\begin{table*}[h]
\centering
\caption{
Ranking performance of DCR compared to baselines across four datasets at a target error rate $\alpha = 0.1$. We evaluate DCR against Oracle CP and the TCPR baseline. Bold values indicate a superior (smaller) relative length for DCR relative to the TCPR method.}
\label{tab:main_results}
\label{tab:main_results}
\setlength{\tabcolsep}{6pt}
\small
\begin{tabular}{ll ccc @{\hskip 15pt} ccc}
\toprule
 & & \multicolumn{6}{c}{\textbf{Metric: Coverage (\%) $\mid$ Relative Length (\%)}} \\
\cmidrule(lr){3-8}
 & & \multicolumn{3}{c}{\textbf{RA Scores}} & \multicolumn{3}{c}{\textbf{VA Scores}} \\
\cmidrule(lr){3-5} \cmidrule(lr){6-8}
\textbf{Dataset} & \textbf{Model} & \textbf{Oracle} & \textbf{TCPR} & \textbf{DCR} & \textbf{Oracle} & \textbf{TCPR} & \textbf{DCR} \\
\midrule
\multirow{3}{*}{Synthetic} & LambdaMART & 90.29$\mid$34.94 & 98.70$\mid$52.52 & 92.37$\mid$\textbf{37.56} & 90.14$\mid$34.52 & 99.40$\mid$54.25 & 92.91$\mid$\textbf{37.65} \\
 & RankNet & 90.65$\mid$22.06 & 99.88$\mid$41.41 & 95.11$\mid$\textbf{26.70} & 90.58$\mid$21.81 & 99.99$\mid$43.84 & 96.01$\mid$\textbf{27.43} \\
 & RankSVM & 90.10$\mid$25.25 & 99.54$\mid$43.47 & 94.54$\mid$\textbf{29.08} & 89.90$\mid$29.86 & 98.87$\mid$52.45 & 93.09$\mid$\textbf{33.86} \\
\midrule
\multirow{3}{*}{Yummly} & LambdaMART & 89.95$\mid$53.33 & 92.50$\mid$57.27 & 89.97$\mid$\textbf{53.36} & 89.89$\mid$53.90 & 93.14$\mid$58.92 & 89.99$\mid$\textbf{54.02} \\
 & RankNet & 90.04$\mid$53.28 & 92.65$\mid$57.35 & 90.03$\mid$\textbf{53.27} & 90.02$\mid$51.47 & 93.07$\mid$56.47 & 90.10$\mid$\textbf{51.58} \\
 & RankSVM & 90.14$\mid$54.31 & 92.57$\mid$58.33 & 90.15$\mid$\textbf{54.33} & 89.98$\mid$54.40 & 93.22$\mid$59.24 & 90.03$\mid$\textbf{54.47} \\
\midrule
\multirow{3}{*}{ESOL} & LambdaMART & 90.18$\mid$55.08 & 94.19$\mid$63.15 & 90.24$\mid$\textbf{55.13} & 90.22$\mid$60.54 & 95.14$\mid$68.19 & 90.07$\mid$\textbf{60.36} \\
 & RankNet & 90.56$\mid$62.63 & 93.83$\mid$69.76 & 90.47$\mid$\textbf{62.41} & 90.33$\mid$65.95 & 94.50$\mid$75.44 & 90.24$\mid$\textbf{65.65} \\
 & RankSVM & 90.36$\mid$57.29 & 94.11$\mid$65.13 & 90.40$\mid$\textbf{57.31} & 89.99$\mid$65.31 & 94.86$\mid$75.51 & 90.35$\mid$\textbf{65.77} \\
\midrule
\multirow{3}{*}{Anime} & LambdaMART & 90.04$\mid$81.23 & 92.38$\mid$84.46 & 90.04$\mid$\textbf{81.23} & 90.00$\mid$81.08 & 92.33$\mid$84.55 & 90.02$\mid$\textbf{81.11} \\
 & RankNet & 89.99$\mid$80.92 & 92.32$\mid$84.23 & 89.99$\mid$\textbf{80.92} & 89.95$\mid$81.42 & 92.32$\mid$85.15 & 89.96$\mid$\textbf{81.43} \\
 & RankSVM & 90.03$\mid$81.15 & 92.30$\mid$84.32 & 90.02$\mid$\textbf{81.13} & 89.98$\mid$81.62 & 92.37$\mid$85.18 & 89.98$\mid$\textbf{81.63} \\
\bottomrule
\end{tabular}
\end{table*}

\paragraph{DCR guarantees the coverage while producing tight prediction sets.} 
Table \ref{tab:main_results} reports the ranking performance of DCR and baseline methods across four datasets. In all evaluated settings, DCR successfully controls the coverage at the target level $1-\alpha=0.9$ while producing prediction sets with significantly smaller relative lengths than those of the TCPR baseline. For instance, on the synthetic dataset using RankNet with the RA score, DCR achieves an empirical coverage of 95.11\% and a relative length of 26.70\%, whereas TCPR yields a substantially larger relative length of 41.41\%. This corresponds to a reduction in set size of approximately 36\% compared to TCPR. Furthermore, DCR frequently achieves performance comparable to the Oracle CP. Specifically, on the Yummly dataset using RankNet with the VA score, DCR achieves an empirical coverage of 90.10\% and a relative length of 51.58\%, which is nearly identical to the 51.47\% length attained by the Oracle CP. This superior efficiency is consistently observed across diverse datasets and models. Overall, the proposed methodology satisfies the target coverage constraint while demonstrating greater efficiency than existing baselines.

\begin{figure}[t]
    \centering
    \includegraphics[width=\linewidth]{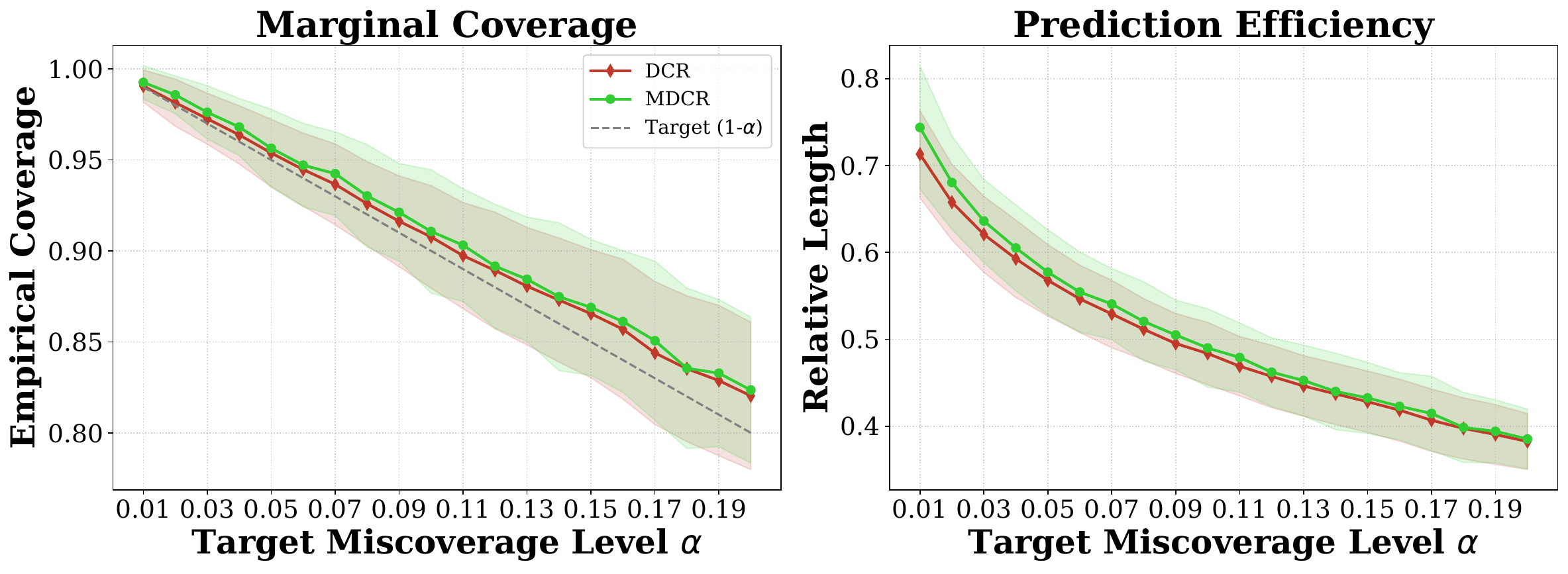} 
    \caption{Empirical coverage and prediction efficiency of RankNet under different target miscoverage levels $\alpha$ on the Synthetic dataset. The results are averaged over 1,000 independent trials with $n=100$ and $m=500$. The shaded areas represent standard deviation.}
    \label{fig:coverage_efficiency}
\end{figure}

\begin{figure}[t]
    \centering
    \includegraphics[width=\linewidth]{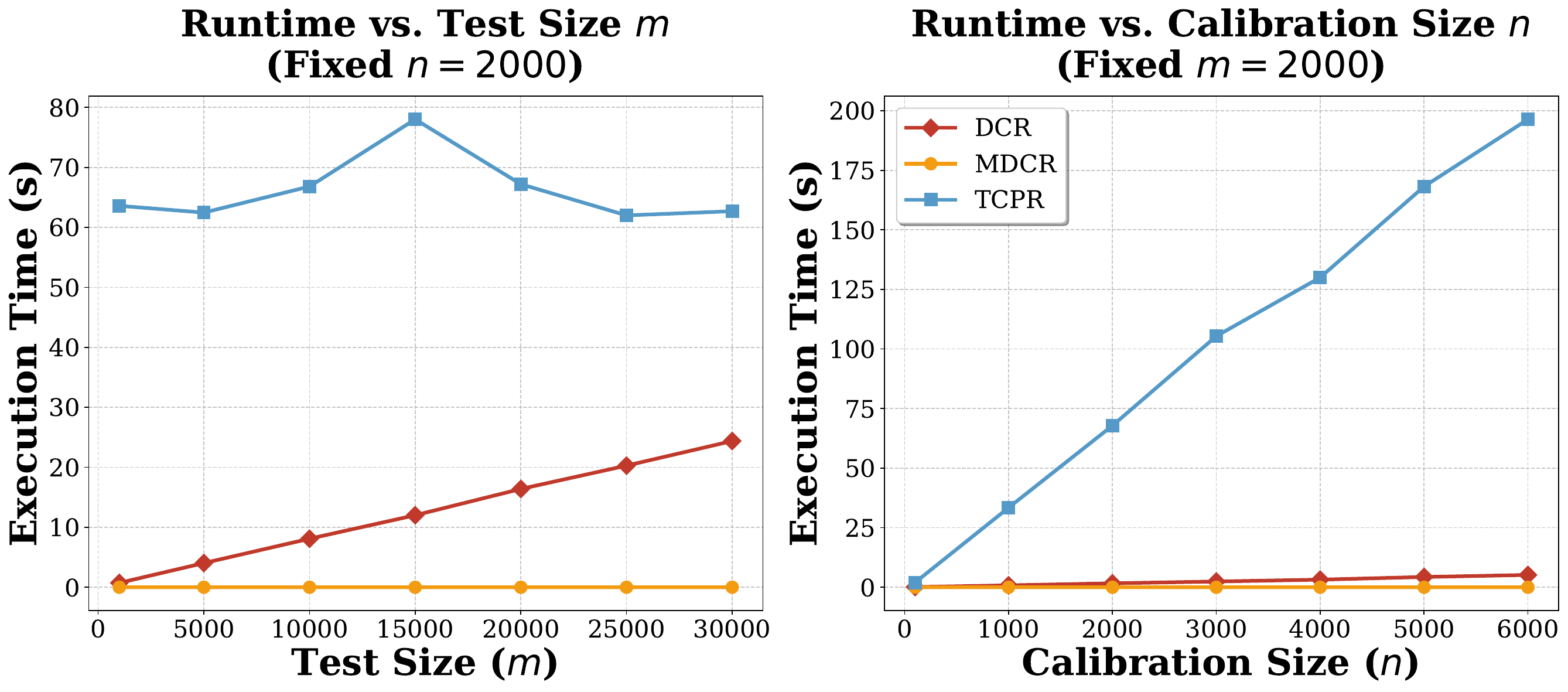} 
    \vspace{-15pt}
    \caption{Runtime comparison of DCR, MDCR, and TCPR on Yummly28k using LambdaMART. MDCR is the most computationally efficient.}
    \label{fig:runtime_analysis}
    \vspace{-10pt}
\end{figure}

\paragraph{DCR achieves higher statistical stability whereas MDCR offers greater computational efficiency.}
Figure~\ref{fig:coverage_efficiency} and Figure~\ref{fig:runtime_analysis} present the performance and runtime results on the Synthetic and Yummly28k datasets. As observed in Figure~\ref{fig:coverage_efficiency}, DCR exhibits slightly lower variance in both realized coverage and prediction set sizes compared to MDCR, as indicated by the narrower shaded regions. However, our extended sensitivity analysis in Appendix~\ref{add:variance} demonstrates that this discrepancy diminishes as the dataset size grows. This implies that for practical applications with sufficient data, MDCR preserves reliability while offering superior scalability. Regarding computation, since DCR complexity scales as $O(nm)$, its runtime will eventually exceed that of TCPR for very large $m$. In contrast, MDCR maintains a highly efficient inference time of $O(n \log n)$, illustrating a favorable trade-off between computational speed and minor statistical fluctuations. Appendix~\ref{add:alpha} provides additional comparisons between DCR and TCPR.



\subsection{Additional analysis}

\begin{figure}[t]
    \centering
    \includegraphics[width=\linewidth]{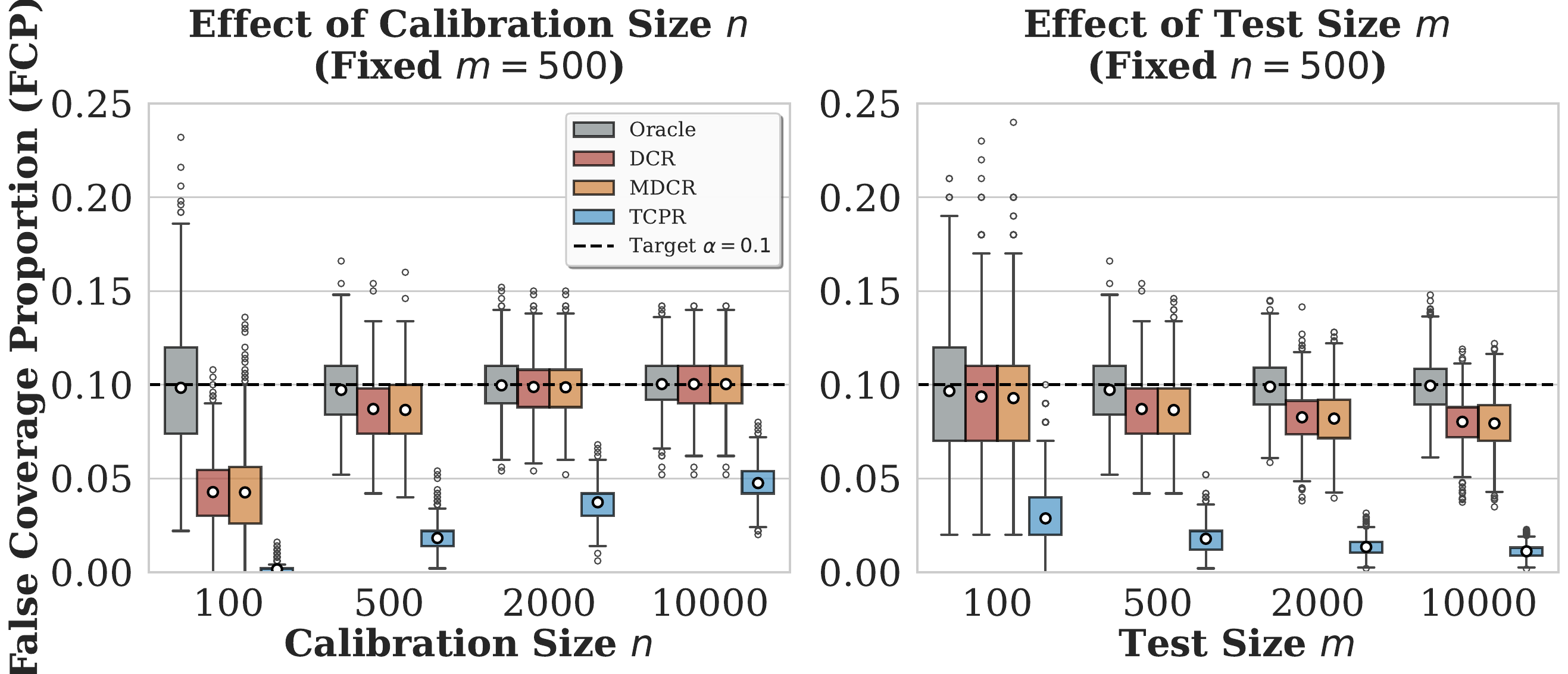}
    \caption{Finite-sample FCP convergence of RankNet on the Synthetic dataset with respect to sample sizes. Results are averaged over 1,000 trials.}
    \label{fig:ablation_fcp}
\end{figure}

\paragraph{Larger calibration set size or smaller test set size yields tighter coverage for DCR.}
Figure~\ref{fig:ablation_fcp} shows the impact of calibration size $n$ and test size $m$ on coverage and relative length. Both DCR and TCPR are conservative when $n$ is small or $m$ is large because the Negative Hypergeometric distribution has high variance under rank uncertainty. In these regimes, the conservatism of DCR is significantly lower than that of TCPR. As $n$ increases or $m$ decreases, DCR and MDCR use distributional information to approach the target $\alpha$, while TCPR remains conservative due to its loose envelope-based bounds. The reduction in FCP dispersion with larger sample sizes validates the concentration guarantees in Theorem~\ref{thm:fcp_bound} and Remark~\ref{rem:asyptotic}. Beyond sample sizes, the precision of the ranking model also influences the degree of conservatism. Appendix~\ref{app:additionalexp} contains an analysis of how DCR adjusts to base model precision $\sigma$ and sheds conservatism as data label noise increases.

\begin{figure}[t]
    \centering
    \resizebox{\linewidth}{!}{
    \begin{subfigure}[t]{0.49\textwidth}
        \centering
        \includegraphics[width=\textwidth]{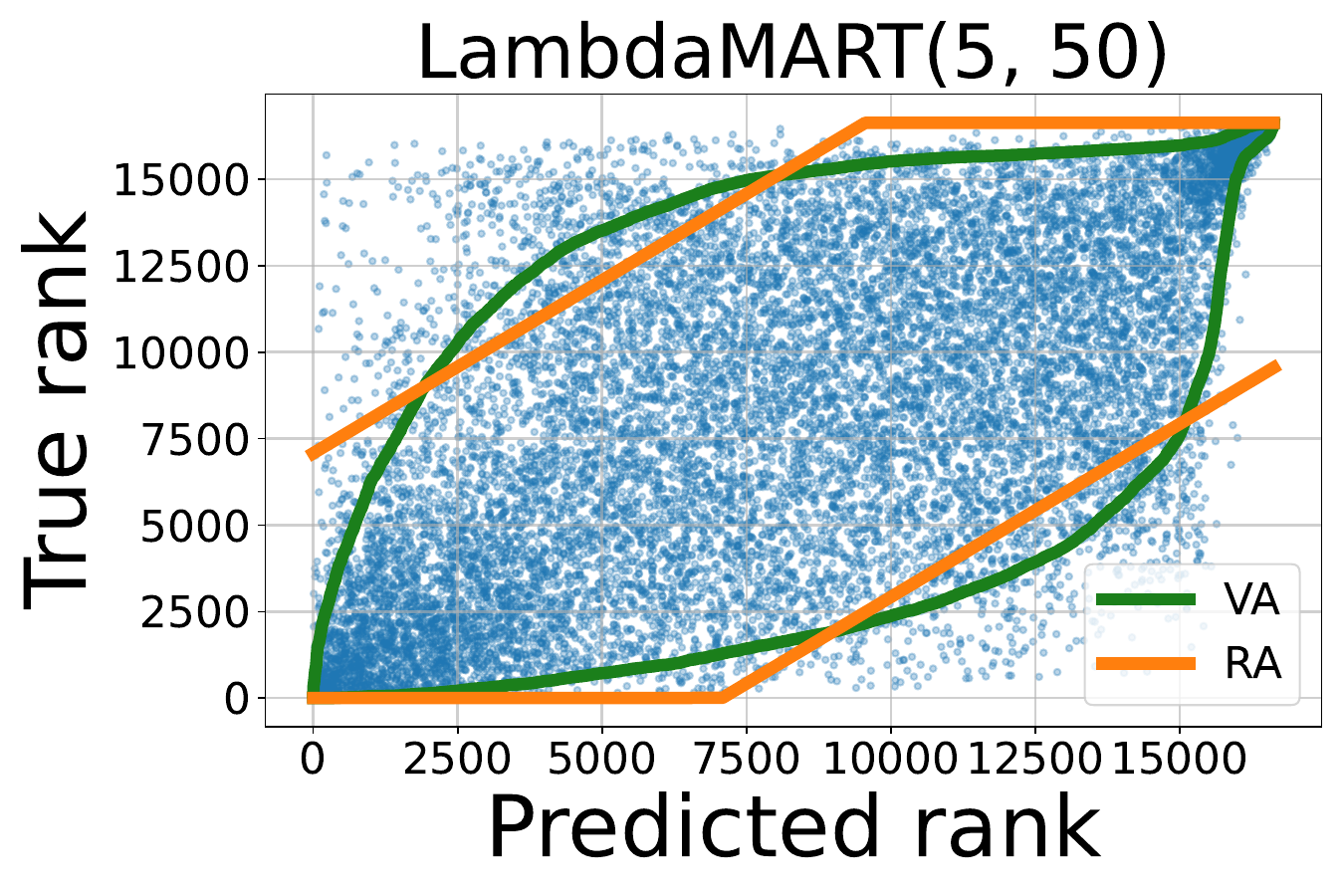}
        \label{fig:anime_weak}
    \end{subfigure}
    \hfill
    \begin{subfigure}[t]{0.49\textwidth}
        \centering
        \includegraphics[width=\textwidth]{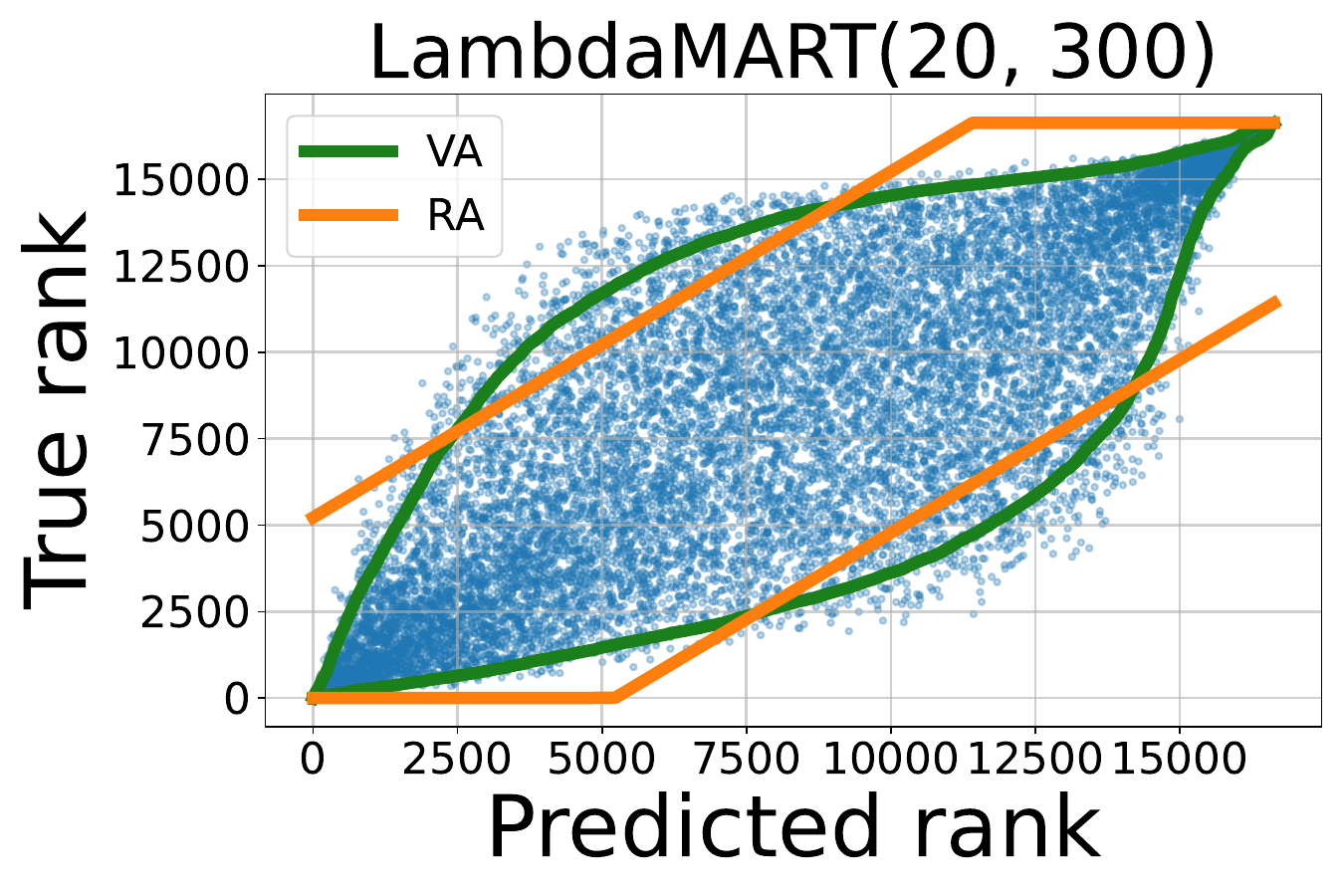}
        \label{fig:anime_strong}
    \end{subfigure}
    }
    \vspace{-10pt}
    \caption{Visualization of prediction intervals on the Anime dataset. The scatter plots display the True Rank versus the Predicted Rank for LambdaMART.
    The curves represent the prediction sets constructed by DCR-RA (orange) and DCR-VA (green) at $\alpha=0.1$.}
    \label{fig:anime_adaptivity}
    \vspace{-10pt}
\end{figure}

\paragraph{VA scores demonstrate superior local adaptivity.} 
Figure \ref{fig:anime_adaptivity} presents True vs. Predicted Rank scatter plots on the Anime dataset across diverse hyperparameter settings. We observe that regardless of these variations, DCR-VA consistently exhibits a characteristic ``spindle'' shape, producing significantly tighter sets at the confident extremes (head and tail) while widening only in the high-uncertainty middle. In contrast, DCR-RA produces fixed-width intervals across the entire ranking, reflecting a homoscedastic assumption that forces global expansion to guarantee coverage for hard-to-rank items. By leveraging value-based information, VA scores effectively capture the heteroscedastic nature of ranking noise. This enables superior efficiency through local adaptivity to model confidence, a property that remains robust across different model capacities.

\section{Related work}
\paragraph{Ranking algorithms and uncertainty quantification.}
Ranking primarily aims to infer a particular full ordering over a given set of items. Most existing research focuses on optimizing the predictive accuracy of this point estimate, utilizing state-of-the-art learning-to-rank methods \citep{liu2009learning} based on pairwise comparisons \citep{shah2018simple, newman2023efficient} or specialized algorithms designed to recover ground-truth rankings from noisy, sparse, or incomplete data \citep{jamieson2011active, fotakis2021aggregating}. While these approaches are highly effective at producing accurate orderings, they typically output a deterministic result without conveying any associated measure of predictive uncertainty. Although some recent works have explored uncertainty quantification (UQ) for full ranking, they predominantly rely on specific parametric families, such as the Bradley--Terry--Luce model \citep{chen2022optimal, karle2023dynamic}. These methods necessitate strong underlying modeling assumptions that are frequently violated in complex real-world scenarios. Consequently, there is a critical and growing need for principled, distribution-free UQ methodologies for full ranking, which motivates our adoption of conformal inference.

\paragraph{Conformal prediction.}
Conformal prediction has emerged as a foundational technique in machine learning \citep{vovk2005algorithmic, angelopoulos2021gentle}, offering rigorous finite-sample coverage guarantees without requiring specific or restrictive distributional assumptions on the data. Our research aligns with the framework of split conformal prediction \citep{papadopoulos2002inductive, lei2018distribution}, a methodology extensively applied to regression \citep{romano2019conformalized}, classification \citep{romano2020classification}, and complex structured outputs such as those from large language models \citep{su2024api}. Recently, these principles have been successfully extended to various ranking tasks, including retrieval-augmented generation \citep{li2024traq, 10.1145/3726302.3730244} and recommendation systems \citep{angelopoulos2023recommendation, liang2024structured}. However, as these specific approaches are generally not directly applicable to the full ranking setting, TCPR \citep{gazin2024transductive} introduced a transductive conformal prediction framework specifically designed for full ranking within a finite population. We adopt this transductive methodology and propose a novel approach that leverages exact distributional insights to achieve higher statistical efficiency and lower computational costs than TCPR.

\section{Conclusion}
\label{sec:conclusion}

In this paper, we propose Distribution-informed Conformal Ranking (DCR), a rigorous methodology that explicitly models the probabilistic structure of rank data to construct efficient prediction sets. To our knowledge, this is the first work to utilize the Negative Hypergeometric distribution to derive the exact distribution of non-conformity scores, bypassing the conservative approximations used in prior research. We theoretically demonstrate that DCR guarantees valid marginal coverage and achieves higher efficiency than existing envelope-based methods. To address scalability, we introduce MDCR, a stochastic variant that reduces computational complexity to $\mathcal{O}(n \log n)$ while preserving validity. Extensive experiments confirm that our methods outperform the state-of-the-art baseline, reducing prediction set sizes by up to 36\% across synthetic and real-world datasets. Overall, DCR enables precise uncertainty quantification for full ranking, and our distributional insights pave the way for future research in complex settings such as covariate shift and uncertainty quantification for top-k rankings.

\section{Impact statement}
This paper presents work whose goal is to advance the field of Machine Learning.
There are many potential societal consequences of our work, none of which we feel must be specifically highlighted here.

\bibliography{example_paper}
\bibliographystyle{icml2026}

\newpage
\appendix
\onecolumn

\section{Conformal prediction}
\label{sec: cp_intro}

Conformal prediction \citep{vovk2005algorithmic, papadopoulos2008inductive} aims to produce prediction sets that contain ground truth labels with a desired coverage rate.
Formally, given a calibration set $\{(X_{i}, Y_{i})\}_{i=1}^n$ in $\mathcal{X} \times \mathcal{Y}$, a predictor $f$ and a user-specified miscoverage rate $\alpha \in (0,1)$, CP aims to construct a prediction set for test data $X_{n+1}$ that satisfy:
\begin{equation}
\label{eq:validity}
    \mathbb{P}(Y_{n+1}\in \mathcal{C}(X_{n+1}))\geq 1-\alpha,  
\end{equation}
First, we compute non-conformity scores $s_i := S(X_{i}, Y_{i})$ on the labeled calibration set, where $S:\mathcal{X} \times \mathcal{Y} \rightarrow \R$ is a score function that measures the strangeness of a sample. 
Based on the scores computed on the calibration set, we search for a threshold $s^*$ such that the probability that a test score falls below $s^*$ does not exceed the user-specified error level $\alpha$.
In particular, we determine the threshold by finding the $(1-\alpha)$-th empirical quantile of the calibration scores:
\begin{equation}\label{eq:threshold}
s^* = \text{Quantile}(\frac{\lceil(n+1)(1-\alpha)\rceil}{n}, \{s_i\}_{i\in[n]})
\end{equation}
Given a test sample $X_{n+1}$, the non-conformity score is computed for each label \(y \in \mathcal{Y}\). 
The prediction set \({\mathcal{C}}(X_{n+1})\) is constructed by including all labels whose non-conformity scores are below the threshold \(s^*\):
\begin{equation} \label{eq:cp_set}
C(X_{n+1}) = \{y \in \mathcal{Y}:S(X_{n+1}, y) \leq s^*\}
\end{equation}
Under mild assumptions, the prediction set ${\mathcal{C}}(X_{n+1})$ in Equation~\eqref{eq:cp_set} has a finite-sample coverage guarantee:

\begin{theorem}\citep{vovk2005algorithmic}\label{prop:finite-sample_coverage_guarantee}
If the calibration samples $\{(X_{i}, Y_{i})\}_{i=1}^n$ and the test sample $(X_{n+1}, Y_{n+1})$ are exchangeable, the prediction set defined in Eq.~\eqref{eq:cp_set} satisfies
\begin{align*}
\mathbb{P}(Y_{n+1}\in \mathcal{C}(X_{n+1}))\geq 1-\alpha
\end{align*}
\end{theorem}

\section{Distribution-free bounds for the absolute ranks of calibration items}
\label{sec:rank_bound}
TCPR \citep{fermanian2025transductive} proposes two categories of absolute rank bounds: theoretical bounds and numerical bounds. 
The theoretical bounds admit explicit closed-form expressions, whereas the numerical bounds rely on Monte Carlo sampling, which can be computationally expensive when the test dataset is large.
The numerical bounds are instantiated through two constructions: linear envelopes and quantile envelopes. 
Below, we describe three approaches proposed in detail.

\paragraph{Theoretical bound}
TCPR relies on Theorem 2.4 of \citet{gazin2024transductive} and derives a well-designed theoretical bound.
Specifically, define $R_i^-$ and $R_i^+$ as:
\begin{equation}
\label{eq:theoretical-envelope}
R_i^\pm
=
R_i^c
+
(m+1)
\left(
\frac{R_i^c}{n}
\pm
\sqrt{
\frac{
\log\!\bigl(4\sqrt{2}\pi\sqrt{\tau_{n,m}/\delta}\bigr)
}{
\tau_{n,m}
}
}
\right),
\qquad
\tau_{n,m} = \frac{nm}{n+m}.
\end{equation}
Under Assumption~\ref{ass:exchange}, the explicit bounds satisfy:
\begin{equation}
\label{eq:theoretical-bound}
\mathbb{P}\!\left(
\forall i \in [n]:
R_i^{c+t} \in [R_i^-, R_i^+]
\right) \ge 1-\delta.
\end{equation}
However, this upper bound is extremely conservative.
To address this, TCPR proposes two numerical bounds.

\paragraph{Linear envelope}
To obtain tighter bounds, a first refinement consists in restricting attention to
rank intervals of the same parametric form as the theoretical bound, but calibrating their width
data-dependently.
Specifically, TCPR considers bounds of the form
\begin{equation}
\label{eq:linear-envelope}
R_i^\pm
=
R_i^c
+
(m+1)
\left(
\frac{R_i^c}{n}
\pm c
\right),
\end{equation}
where $c > 0$ is a scalar parameter determined by the linear envelope algorithm.
The parameter $c$ is selected so that the resulting envelope contains the absolute ranks
of all calibration items with probability at least $1-\delta$.
To achieve this, we exploit the fact that the joint distribution of the ordered vector
$(R_{(1)}^{c+t}, \dots, R_{(n)}^{c+t})$ is universal under exchangeability.
Thus, we can simulate $(R_{(1)}^{c+t}, \dots, R_{(n)}^{c+t})$ using a Monte-Carlo approach as detailed in Algorithm \ref{alg:simRtestcal}.

\begin{algorithm}[H]
    \caption{Simulation of $R^{c+t}$}
    \label{alg:simRtestcal}
    \begin{algorithmic}[1]
        \STATE Draw $n + m$ uniform random variable $U_i$ on $[0,1]$.
        \FOR{$i = 1, \ldots, n$}
        \STATE $\tilde{R}^{c+t}_i \gets \text{R}(U_i, \{U_j\}_{j\in[n+m]} )$
        \ENDFOR
        \STATE \textbf{Output:} Sort $\tilde{R}^{c+t}$
    \end{algorithmic}
\end{algorithm}

By repeatedly simulating this distribution, we choose the smallest value of $c$
such that
\begin{equation}
\label{eq:linear-envelope-prob}
\mathbb{P}\!\left(
\forall i \in [n]:
R_i^{c+t} \in [R_i^-, R_i^+]
\right) \ge 1-\delta .
\end{equation}
The full algorithm of linear envelope is detailed in Algorithm \ref{alg:linear_env}.
Compared to the theoretical bound, the linear envelope typically yields substantially
narrower intervals.
However, because the same width parameter $c$ is used for all ranks, the resulting bounds
may still be suboptimal for specific ranks, which motivates the quantile envelope.

\begin{algorithm}[H]
\caption{Linear envelope procedure}
\label{alg:linear_env}
\begin{algorithmic}[1]
\STATE \textbf{Input:} $n, m, K,$ and $\delta \in (0,1)$, $(R_i^{c})_{i \in [n]}$

\FOR{$k = 1, \ldots, K$}
    \STATE $\tilde{R}^{(k)} \gets$ output of \textbf{Algorithm~\ref{alg:simRtestcal}}
\ENDFOR

\STATE $\tilde{c} \gets \inf \Bigl\{ c > 0 :
\frac{1}{K}\sum_{k=1}^K
1\!\left(
\exists i\in[n] :
\tilde{R}_i^{(k)} \notin
\bigl[
R_i^c + (m+1)(R_i^c/n - c),\,
R_i^c + (m+1)(R_i^c/n + c)
\bigr]
\right)
\le \delta
\Bigr\}$

\FOR{$i = 1, \ldots, n$}
    \STATE $\widehat{R}_i^- \gets R_i^c + (m+1)\bigl(R_i^c/n - \tilde{c}\bigr)$
    \STATE $\widehat{R}_i^+ \gets R_i^c + (m+1)\bigl(R_i^c/n + \tilde{c}\bigr)$
\ENDFOR

\STATE \textbf{Output:} $\widehat{R}_{R_i^c}^-, \widehat{R}_{R_i^c}^+$ for all $i \in [n]$.
\end{algorithmic}
\end{algorithm}

\paragraph{Quantile envelope}
To further improve adaptivity, TCPR introduces an envelope approach based on rank-wise quantiles.
For each relative calibration rank $r \in [n]$, let $R^{c+t}(r)$ denote the absolute rank of the calibration
item with relative rank $r$.
Under exchangeability, the distribution of $R^{c+t}(r)$ is universal and can be approximated
by Monte Carlo simulation.

For a given parameter $\gamma \in (0,1/2)$, we define
\begin{equation}
\label{eq:quantile-envelope}
R_r^- = \text{Quantile}(\gamma,R^{c+t}(r)\bigr),
\qquad
R_r^+ = \text{Quantile}(1-\gamma,R^{c+t}(r)\bigr),
\end{equation}
where $Q_u(\cdot)$ denotes the $u$-th quantile of the simulated distribution.
The parameter $\gamma$ is chosen as large as possible subject to the global coverage constraint
\begin{equation}
\label{eq:quantile-envelope-prob}
\mathbb{P}\!\left(
\forall i \in [n]:
R_i^{c+t} \in [R_{R_i^c}^-, R_{R_i^c}^+]
\right) \ge 1-\delta .
\end{equation}
The full algorithm of the quantile envelope is detailed in Algorithm \ref{alg:quantile_env}.
Unlike the linear envelope, the quantile envelope adapts to the variability of the absolute rank
distribution at each calibration rank.
This produces tighter bounds than linear envelope, but still leads to conservative coverage in practice.

\begin{algorithm}[H]
	\caption{Quantile envelope procedure}
	\label{alg:quantile_env}
	\begin{algorithmic}[1]
		\STATE \textbf{Input:} $n, m, K,$ and $\delta \in (0,1)$, $(R_i^{c})_{i \in [n]}$
		\FOR{$k = 1, \ldots, K$}
		\STATE $\tilde{R}^{(k)} \gets$ Output of \textbf{Algorithm \ref{alg:simRtestcal}}
		\ENDFOR
		\STATE $\tilde{\gamma} \gets \max\gamma$ 

		such that:
		$\sum_k 1 \{\exists\,i: \tilde{R}_{i}^{(k)} \notin [\text{Quantile}(\gamma,(\tilde{R}_i^{(l)})), \text{Quantile}(1-\gamma,(\tilde{R}_i^{(l)}))]\}$ 
		$\leq K\delta$
		
		\FOR{$j = 1, \ldots, n$}
		\STATE $\widehat{R}^{-}_j \gets  \text{Quantile}(\tilde{\gamma},(\tilde{R}_i^{(l)})_{l\in[K]})$
		\STATE $\widehat{R}^{+}_j \gets  \text{Quantile}(1-\tilde{\gamma},((\tilde{R}_i^{(l)})_{l\in[K]})$
		\ENDFOR
		\STATE \textbf{Output:} $\widehat{R}_{R_i^c}^{-}, \widehat{R}_{R_i^{c}}^{+}$ for $i \in [n]$.
	\end{algorithmic}
\end{algorithm}




\section{Details of Datasets and Models}
\label{app:details}

\subsection{Datasets}

\paragraph{Experimental setup for Figure \ref{fig:tcpr_vs_oracle}}
We follow the data generation process of \citet{pedregosa2012learning, gazin2024transductive}, with minor modifications to better match our experimental setting. Specifically, we generate a total of $n_{\text{train}} + n + m$ i.i.d.\ samples $(X_i, Y_i)$, where $n_{train}=10000$, $n=200$, and $m=300$.
Each covariate vector $X_i \in \mathbb{R}^{d}$ is drawn from a standard Gaussian distribution, $X_i \sim \mathcal{N}(0, I_d)$, with dimension $d = 10$. The response variable is generated according to
\[
Y_i = \bigl(1 + \exp(-w^\top X_i)\bigr)^{-1} + \varepsilon_i,
\]
where $w = (1,\ldots,1) \in \mathbb{R}^{d}$ and the noise terms $\varepsilon_i$ are independent and distributed as $\mathcal{N}(0, 0.01)$.

For each observation, the rank $R_i$ is defined as the rank of $Y_i$ among all sampled responses. 
We first use the $n_{\text{train}}$ samples to train a RankNet model.
The pre-trained RankNet is then used to apply TCPR and oracle CP on the remaining $n + m$ items.

\paragraph{Synthetic Data.}
To rigorously evaluate the proposed methods under controlled conditions, we generate synthetic data using a linear latent model with additive noise. For each instance $i$, the feature vector $x_i \in \mathbb{R}^d$ is sampled from a standard multivariate Gaussian distribution $\mathcal{N}(0, I_d)$. The ground-truth score $y_i$ is generated as:
\begin{equation}
    y_i = x_i^\top w + \epsilon_i, \quad \epsilon_i \sim \mathcal{N}(0, \sigma^2),
\end{equation}
where $w \in \mathbb{R}^d$ is a fixed weight vector sampled from $\mathcal{N}(0, 1)$ and normalized to unit length ($\|w\|_2 = 1$). In our experiments, we set the feature dimension $d=20$ and the noise level $\sigma=0.2$. The absolute ranks $R_i^{c+t}$ are then derived by sorting these scores $y_i$ across the combined calibration and test sets.

\paragraph{Yummly28k.}
This dataset \citep{wilber2015learning} contains 27,638 recipes with metadata. We use the raw ingredient lines to simulate a preference learning task. Specifically, we extract TF-IDF features ($d=101$) from the ingredient text. To simulate a ranking task, we define a target preference vector $x^*$ and calculate the ground-truth scores as the negative Euclidean distance to this target: $y_i = -\|x_i - x^*\|_2$. This simulates a scenario where items are ranked based on their similarity to a user's ideal recipe profile.

\paragraph{ESOL.}
The ESOL dataset \citep{delaney2004esol} is a standard benchmark in chemoinformatics, containing the water solubility (log solubility in mols per litre) of 1,128 molecular compounds. Each molecule is represented by 1,024-bit Morgan Fingerprints. We treat this as a ranking problem where the goal is to rank molecules based on their solubility, which is critical for drug discovery and lead optimization.

\paragraph{Anime Recommendation.}
This dataset is a large-scale Learning-to-Rank (LTR) benchmark consisting of 16,681 movies and 15,163 users. Each (user, movie) tuple is associated with a rating from 0 to 10. Following the transductive setting, we aim to quantify the uncertainty of a target model relative to a high-capacity reference model. The reference model is a large LambdaMART ensemble (400 trees), which defines the ground-truth ordering $R_i^{c+t}$. The features include both item characteristics and user-movie interactions.

\subsection{Model Architectures and Hyperparameters}

\paragraph{RankNet.}
In our implementation, RankNet is configured as a deep regression network to generate continuous scoring functions. The architecture consists of a 5-layer fully connected MLP. Each hidden layer contains 128 neurons, integrated with Batch Normalization, ReLU activation, and a Dropout layer ($p=0.2$) to enhance generalization. The model is trained to minimize the Mean Squared Error (MSE) loss using the AdamW optimizer with a starting learning rate of $10^{-3}$ and weight decay of $10^{-4}$. We employ the \texttt{ReduceLROnPlateau} scheduler to dynamically adjust the learning rate during training. To ensure convergence stability, we utilize a batch size of 128 and implement an early stopping mechanism with a patience of 15 epochs.

\paragraph{LambdaMART.}
We utilize the LightGBM implementation of LambdaMART configured in regression mode to handle continuous label distributions across diverse datasets. The model uses Mean Squared Error (MSE) as the primary training metric. Key hyperparameters include a learning rate of 0.03 and a high-capacity tree structure with 127 leaves per tree. To prevent overfitting while maintaining predictive power, we set both the feature fraction and bagging fraction to 0.9, with bagging performed every 5 iterations. The training process incorporates an early stopping patience of 50 rounds based on validation performance.

\paragraph{RankSVM.}
Our RankSVM is implemented using Kernel Ridge Regression (KRR) to effectively capture non-linear ranking patterns. We employ the Random Kitchen Sinks method via \texttt{RBFSampler} to approximate the Radial Basis Function (RBF) kernel, mapping input features into a 2,000-dimensional randomized feature space with a kernel scale $\gamma=0.01$. A Ridge regression model with a regularization parameter $\alpha=1.0$ is then fitted on these mapped features. This kernelized approach allows the model to produce stable and expressive scoring functions that capture the underlying ranking uncertainty more effectively than a standard linear SVM.

\section{Proofs}
\label{app:proofs}

\subsection{Proof of Proposition \ref{prop:nh}}
\label{app:proof_nh}

\begin{proof}
Let $N = n + m$ be the total number of items. Under the exchangeability assumption, the joint distribution of the ranks is invariant under permutation. Consider the sorted sequence of all $N$ items, denoted as $z_{(1)} < z_{(2)} < \dots < z_{(N)}$. We can view the assignment of items to either the calibration set $\mathcal{I}_{cal}$ or the test set $\mathcal{I}_{test}$ as a process of drawing $N$ labels without replacement, where $n$ labels correspond to ``Calibration" and $m$ labels correspond to ``Test". Due to exchangeability, all $\binom{N}{n}$ possible subsets of positions for the calibration items are equally likely.

Let $r = R_i^c$ be the observed relative rank of item $i$ within the calibration set. This implies that item $i$ is the $r$-th calibration item encountered in the sorted sequence $z_{(1)}, \dots, z_{(N)}$. The variable $R_i^t$ counts the number of test items appearing before item $i$ in this sorted sequence. Let $k$ be a realization of $R_i^t$. The event $\{R_i^t = k \mid R_i^c = r\}$ occurs if and only if the item $i$ is located at the absolute position $r+k$ in the sorted sequence.

For this to happen, the arrangement of labels must satisfy specific conditions. First, the position $r+k$ is fixed as the $r$-th calibration item. Second, among the preceding $r+k-1$ positions, there must be exactly $r-1$ calibration items (and thus $k$ test items). The number of ways to choose these positions is $\binom{r+k-1}{r-1} = \binom{r+k-1}{k}$. Third, the remaining $N - (r+k)$ positions must contain the remaining $n-r$ calibration items (and $m-k$ test items). The number of ways to arrange this suffix is $\binom{N-r-k}{n-r} = \binom{N-r-k}{m-k}$.

The probability is the ratio of favorable arrangements to the total number of ways to choose $n$ calibration positions from $N$ spots:
\begin{equation}
    \mathbb{P}(R_i^t = k \mid R_i^c = r) = \frac{\binom{r+k-1}{k} \binom{N-r-k}{m-k}}{\binom{N}{m}}, \quad k \in \{0, \dots, m\}.
\end{equation}
This is precisely the probability mass function of the Negative Hypergeometric distribution with parameters $N$, $m$, and $r$.
\end{proof}

\subsection{Proof of Theorem \ref{thm:validity}}
\label{app:proof_validity}

\begin{lemma}[Score Exchangeability]
\label{lem:score_exch}
Suppose the data $Z = \{(X_i, Y_i)\}_{i=1}^{N}$ are exchangeable. Then the sequence of non-conformity scores $S=(S_1,\dots,S_N)$ is finitely exchangeable. 
\end{lemma}

\begin{proof}
Let $Z_i = (X_i, Y_i)$ for $i=1, \dots, N$, where $N=n+m$. Let $\mathbf{Z} = (Z_1, \dots, Z_N)$ denote the full dataset.
Let $\Phi: (\mathcal{X} \times \mathbb{R})^N \to \mathbb{R}^N$ be the mapping that transforms the raw dataset into the vector of non-conformity scores $\mathbf{S} = (S_1, \dots, S_N)$. Specifically, the $i$-th component is given by:
\begin{equation}
    S_i = \Phi_i(\mathbf{Z}) = s\left( \text{Rank}(Y_i, \{Y_j\}_{j=1}^N), A(X_i, \{X_j\}_{j=1}^N) \right).
\end{equation}
Since $\mathbf{Z}$ is exchangeable, $\mathbf{Z} \overset{d}{=} \mathbf{Z}_\pi$ for any permutation $\pi$, where $\mathbf{Z}_\pi = (Z_{\pi(1)}, \dots, Z_{\pi(N)})$.
The ranking operation and score function depend only on the values and relative order of items, not their indices. Thus, the mapping $\Phi$ is permutation-equivariant. Specifically, if we permute the input data, the output scores are permuted by the same indices:
\begin{equation}
    \Phi(\mathbf{Z}_\pi) = (S_{\pi(1)}, \dots, S_{\pi(N)}).
\end{equation}
Combining these properties:
\begin{equation}
    (S_1, \dots, S_N) = \Phi(\mathbf{Z}) \overset{d}{=} \Phi(\mathbf{Z}_\pi) = (S_{\pi(1)}, \dots, S_{\pi(N)}).
\end{equation}
Thus, the joint distribution of scores is invariant under permutation, establishing finite exchangeability.
\end{proof}

\begin{proof}[Proof of Theorem \ref{thm:validity}]
Consider a random index $I$ sampled uniformly from the calibration indices $\{1, \dots, n\}$, independent of the data. The random variable $S_I^{true}$ represents the score of this randomly selected calibration item. Its marginal CDF is derived as follows:
\begin{align}
    \Prob(S_I^{true} \le s) &= \sum_{j=1}^n \Prob(S_I^{true} \le s \mid I=j) \cdot \Prob(I=j) \\
    &= \sum_{j=1}^n \Prob(S_j^{true} \le s) \cdot \frac{1}{n} \\
    &= \frac{1}{n} \sum_{j=1}^n F_j(s) \\
    &= F_{mix}(s).
\end{align}
Thus, $F_{mix}(s)$ is precisely the marginal cumulative probability that a randomly chosen calibration score falls below $s$.

From the exchangeability established in Lemma \ref{lem:score_exch}, the score of a new test item $S_{n+j}^{true}$ ($j \in \{1, \dots, m\}$) follows the same marginal distribution as $S_I^{true}$. Thus:
\begin{equation}
    S_{n+j}^{true} \overset{d}{=} S_I^{true} \sim F_{mix}.
\end{equation}
The CR algorithm selects the threshold $s^*$ such that $F_{mix}(s^*) \ge 1 - \alpha$. By the definition of the quantile function:
\begin{equation}
    \Prob(S_{n+j}^{true} \le s^*) = \Prob(S_I^{true} \le s^*) = F_{mix}(s^*) \ge 1 - \alpha.
\end{equation}
This guarantees that the prediction set $\hat{\calC}_j$ contains the true rank $R_j^{c+t}$ with probability at least $1-\alpha$.
\end{proof}

\subsection{Proof of Theorem \ref{thm:validity_mDCR}}
\label{app:proof_mecp_validity}

\begin{proof}
The validity of the MECP algorithm is established by demonstrating that the sequence composed of the simulated calibration scores and the true (unobserved) test score is exchangeable. This property allows for the application of standard conformal prediction guarantees.

Consider the probability space induced jointly by the data generating process and the algorithm's internal randomization. Let $\mathbf{Z} = \{Z_1, \dots, Z_{n+m}\}$ denote the exchangeable dataset drawn from $\mathcal{P}_{XY}$, where $Z_k = (X_k, Y_k)$. For a specific test item $j$, the true non-conformity score is given by $S_{n+j} = s(R_{n+j}^c + R_{n+j}^t, \hat{R}_{n+j}^{c+t})$, where $R_{n+j}^t$ is the unobserved count of other test items ranked lower than $Z_{n+j}$. By Proposition \ref{prop:nh}, the conditional distribution of this rank component is $R_{n+j}^t \mid R_{n+j}^c \sim \text{NegHypergeometric}(n+m, m, R_{n+j}^c)$.

Conversely, for each calibration item $i \in \{1, \dots, n\}$, the MECP algorithm computes a proxy score $\tilde{S}_i = s(R_i^c + \tilde{R}_i^t, \hat{R}_i^{c+t})$. Here, $\tilde{R}_i^t$ is a random variable explicitly sampled by the algorithm from the same distribution: $\tilde{R}_i^t \sim \text{NegHypergeometric}(n+m, m, R_i^c)$.

Due to the exchangeability of the original data $\mathbf{Z}$, the marginal distributions of the variables $Z_i$ and $Z_{n+j}$ are identical, implying that their relative ranks $R_i^c$ and $R_{n+j}^c$ are identically distributed. Consequently, the pairs consisting of the data point and its associated rank increment (whether real or simulated) share the same joint distribution:
\begin{equation}
    (Z_i, \tilde{R}_i^t) \overset{d}{=} (Z_{n+j}, R_{n+j}^t).
\end{equation}
It follows that the computed proxy score $\tilde{S}_i$ and the true test score $S_{n+j}$ are identically distributed, i.e., $\tilde{S}_i \overset{d}{=} S_{n+j}$.

Furthermore, since the internal random variables $\{\tilde{R}_i^t\}_{i=1}^n$ are generated independently conditional on the data, the combined sequence of scores $V = (\tilde{S}_1, \dots, \tilde{S}_n, S_{n+j})$ forms an exchangeable sequence. We calculate the threshold $\hat{s}^*$ as the $\lceil(n+1)(1-\alpha)\rceil / (n+1)$-quantile of the set $\{\tilde{S}_1, \dots, \tilde{S}_n\} \cup \{\infty\}$. Applying the fundamental validity property of conformal prediction for exchangeable scalar scores, we obtain:
\begin{equation}
    \Prob(S_{n+j} \le \hat{s}^*) \ge 1 - \alpha.
\end{equation}
Since the event $R_{n+j}^{c+t} \in \widehat{\mathcal{C}}_{j}$ is equivalent to $S_{n+j} \le \hat{s}^*$, this concludes the proof.
\end{proof}

\subsection{Proof of Theorem \ref{thm:efficiency}}
\label{app:proof_efficiency}

\begin{lemma}\label{lemma:A2}
Let $F_i$ be the CDF of the true non-conformity score $S_i^{true}$ for the $i$-th calibration item (as modeled by DCR). Let $S_i^{TCPR}$ be the proxy score defined in TCPR using an envelope $[R_i^-, R_i^+]$ with coverage probability $1-\delta$. Then,
\[
S_i^{TCPR} \ge \text{Quantile}(1-\delta, F_i).
\]
\end{lemma}

\begin{proof}
Conditioned on the observed data, consider the randomness of the unobserved absolute rank $R_i^{c+t}$. By the definition of the TCPR proxy score, for any realization $R \in [R_i^-, R_i^+]$, we have $s(R, \hat{R}_i^{c+t}) \le S_i^{TCPR}$.
This implies the following inclusion of events:
$$
\{ R_i^{c+t} \in [R_i^-, R_i^+] \} \subseteq \{ S_i^{true} \le S_i^{TCPR} \}.
$$
Applying the probability measure to both sides:
$$
F_i(S_i^{TCPR}) = \mathbb{P}(S_i^{true} \le S_i^{TCPR}) \ge \mathbb{P}(R_i^{c+t} \in [R_i^-, R_i^+]).
$$
Since the TCPR envelope is constructed to satisfy marginal coverage of at least $1-\delta$ (i.e., $\mathbb{P}(R_i^{c+t} \in [R_i^-, R_i^+]) \ge 1-\delta$), we have:
$$
F_i(S_i^{TCPR}) \ge 1 - \delta.
$$
By the definition of the quantile function $Q(p) = \inf \{x : F(x) \ge p\}$, the inequality $F_i(S_i^{TCPR}) \ge 1-\delta$ directly implies:
$$
S_i^{TCPR} \ge \text{Quantile}(1-\delta, F_i).
$$
\end{proof}

\begin{lemma}\label{lemma:A1}
Let $F_1, \dots, F_n$ be cumulative distribution functions of discrete distributions, and let $F = \frac{1}{n} \sum_{i=1}^n F_i$. For constants $0 < \delta < \alpha < 1$, define the component quantiles $p_i = \text{Quantile}(1-\delta, F_i)$ and the mixture quantile $q = \text{Quantile}(1-\alpha, F)$. Let $T = \text{Quantile}(1-\alpha+\delta, \{p_i\}_{i=1}^n)$, defined as the $\lceil n(1-\alpha+\delta) \rceil$-th order statistic of $\{p_i\}$. Then,
\[
q \le T.
\]
\end{lemma}

\begin{proof}
Let $S = \{i : p_i \le T\}$. By the definition of $T$ as the $\lceil n(1-\alpha+\delta) \rceil$-th smallest value in $\{p_i\}$, the size of set $S$ satisfies $|S| \ge n(1-\alpha+\delta)$.

Consider the value of the mixture CDF at $T$:
\[
F(T) = \frac{1}{n} \sum_{i=1}^n F_i(T) \ge \frac{1}{n} \sum_{i \in S} F_i(T).
\]
For any $i \in S$, we have $p_i \le T$. By the monotonicity of CDFs and the definition of the quantile $p_i$, $F_i(T) \ge F_i(p_i) \ge 1-\delta$. Substituting this into the sum:
\[
F(T) \ge \frac{|S|}{n}(1-\delta) \ge (1-\alpha+\delta)(1-\delta).
\]
Expanding the right-hand side yields:
\[
(1-\alpha+\delta)(1-\delta) = 1 - \delta - \alpha + \alpha\delta + \delta - \delta^2 = 1 - \alpha + \delta(\alpha-\delta).
\]
Since $0 < \delta < \alpha$, we have $\delta(\alpha-\delta) > 0$. Consequently:
\[
F(T) > 1-\alpha.
\]
By the definition of the mixture quantile $q = \inf \{x : F(x) \ge 1-\alpha\}$, the condition $F(T) > 1-\alpha$ implies $q \le T$.
\end{proof}

\begin{proof}[Proof of Theorem \ref{thm:efficiency}]
Let $F_i$ be the distribution of $S_i^{true}$ and $F_{mix}=\frac{1}{n}\sum_{i=1}^nF_i$. Define $p_i = \text{Quantile}(1-\delta, F_i)$.

By Lemma~\ref{lemma:A2}, we have $S_i^{TCPR} \ge p_i$ for all $i \in [n]$.
Since the empirical quantile function is monotonic with respect to the input values, we have:
\begin{equation}
\begin{aligned}
s^*_{TCPR} &= \text{Quantile}(1-\alpha + \delta, \{S_i^{TCPR}\}_{i=1}^n) \\
&\ge \text{Quantile}(1-\alpha + \delta, \{p_i\}_{i=1}^n) \quad\text{(By Lemma~\ref{lemma:A2})}\\
&\ge \text{Quantile}(1-\alpha, F_{mix}) \quad\text{(By Lemma~\ref{lemma:A1})}\\
&= s^*_{DCR}.
\end{aligned}
\end{equation}
Thus, $s^*_{DCR} \le s^*_{TCPR}$. Since the prediction set size is directly determined by the threshold, this implies $|\hat{\mathcal{C}}^{DCR}| \le |\hat{\mathcal{C}}^{TCPR}|$.
\end{proof}

\subsection{Proof of Theorem~\ref{thm:fcp_bound}}
\label{app:proof_fcp}

We model the calibration and test sets as a random partition of the exchangeable sequence $\mathbf{S}=(S_1,\dots,S_N)$. Conditioned on the multiset of values $\{S_1,\dots,S_N\}$, the calibration index set $\mathcal{I}_{\mathrm{cal}}$ is drawn uniformly at random without replacement from $\{1,\dots,N\}$ (and $\mathcal{I}_{\mathrm{test}}$ is its complement).

\begin{lemma}[DKW-type inequality for finite-population empirical CDFs under sampling without replacement]
\label{lemma:dkw_swr}
Let $\mathcal{X}_N=\{x_1,\dots,x_N\}$ be a fixed finite population and let $X_1,\dots,X_n$ be drawn uniformly without replacement from $\mathcal{X}_N$. Let $\hat{F}_n$ denote the empirical CDF of the sample and let $F_N$ denote the population CDF.
Then there exists a universal constant $C>0$ such that for any $\epsilon>0$,
\begin{equation}
\label{eq:dkw_swr}
\mathbb{P}\!\left(\sup_{t\in\mathbb{R}} \bigl|\hat{F}_n(t)-F_N(t)\bigr|>\epsilon\right)
\le 2\exp(-C\,n\,\epsilon^2).
\end{equation}
\end{lemma}

\begin{proof}
The statement is a standard consequence of classical finite-population empirical process theory.
In particular, results on uniform deviations of the empirical distribution function under sampling without replacement
(e.g., \citet{rosen1965} and \citet{serfling1974}) imply that the Kolmogorov--Smirnov statistic
$\sup_t |\hat{F}_n(t)-F_N(t)|$ admits a sub-Gaussian tail bound of the form~\eqref{eq:dkw_swr}
with a universal constant $C>0$ (independent of $N,n$ and the population values).
For completeness, we use this established DKW-type inequality for finite populations as stated.
\end{proof}

\begin{proof}[Proof of Theorem~\ref{thm:fcp_bound}]
Recall that the conformal threshold $s^*$ is defined as the empirical quantile satisfying
\begin{equation}
\hat{F}_{\mathrm{cal}}(s^*) \ge \frac{\lceil(n+1)(1-\alpha)\rceil}{n+1} \ge 1-\alpha,
\end{equation}
and the realized False Coverage Proportion (FCP) on the test set is
$\mathrm{FCP}=1-\hat{F}_{\mathrm{test}}(s^*)$.

We aim to bound $\mathrm{FCP}-\alpha$. Since $1-\alpha \le \hat{F}_{\mathrm{cal}}(s^*)$, we have
\begin{align}
\mathrm{FCP}-\alpha
&= 1-\hat{F}_{\mathrm{test}}(s^*)-\alpha \\
&\le \hat{F}_{\mathrm{cal}}(s^*)-\hat{F}_{\mathrm{test}}(s^*).
\end{align}
(The inequality is exact; any discretization effect from the $(n\!+\!1)^{-1}$ grid can be tracked separately and is dominated by the concentration terms for large $n$.)

Let $F_N$ be the population CDF of the full multiset $\{S_1,\dots,S_N\}$. Using $F_N$ as a pivot,
\begin{align}
\hat{F}_{\mathrm{cal}}(s^*)-\hat{F}_{\mathrm{test}}(s^*)
&= \bigl(\hat{F}_{\mathrm{cal}}(s^*)-F_N(s^*)\bigr) + \bigl(F_N(s^*)-\hat{F}_{\mathrm{test}}(s^*)\bigr) \\
&\le \sup_t \bigl|\hat{F}_{\mathrm{cal}}(t)-F_N(t)\bigr| \;+\; \sup_t \bigl|\hat{F}_{\mathrm{test}}(t)-F_N(t)\bigr|.
\end{align}

We now apply Lemma~\ref{lemma:dkw_swr} to each supremum term.
Allocate failure probability $\beta/2$ to each deviation and define
\begin{equation}
\epsilon_n := \sqrt{\frac{\ln(4/\beta)}{C\,n}}
\qquad\text{and}\qquad
\epsilon_m := \sqrt{\frac{\ln(4/\beta)}{C\,m}}.
\end{equation}
By Lemma~\ref{lemma:dkw_swr},
\begin{align}
\mathbb{P}\!\left(\sup_t \bigl|\hat{F}_{\mathrm{cal}}(t)-F_N(t)\bigr|>\epsilon_n\right)
&\le 2\exp(-C n \epsilon_n^2)=\frac{\beta}{2}, \\
\mathbb{P}\!\left(\sup_t \bigl|\hat{F}_{\mathrm{test}}(t)-F_N(t)\bigr|>\epsilon_m\right)
&\le 2\exp(-C m \epsilon_m^2)=\frac{\beta}{2}.
\end{align}
By the union bound, with probability at least $1-\beta$, both events fail to occur simultaneously, hence
\begin{equation}
\mathrm{FCP}-\alpha \le \epsilon_n+\epsilon_m.
\end{equation}
This rearranges to the stated bound in Theorem~\ref{thm:fcp_bound}.
\end{proof}

\subsection{Proof of Theorem~\ref{thm:variance} (Asymptotic Variance Reduction)}
\label{app:proof_variance}
\begin{proof}
Let $\mathcal{D}_{\text{cal}} = \{(X_i, R_i^c)\}_{i=1}^n$ denote the calibration data.
DCR computes the threshold $s^*_{\text{DCR}}$ deterministically based on the mixture distribution $F_{\text{mix}}(t)$. Thus, conditioned on $\mathcal{D}_{\text{cal}}$, $s^*_{\text{DCR}}$ is a constant, and its conditional variance is zero:
\[
\Var(s^*_{\text{DCR}} \mid \mathcal{D}_{\text{cal}}) = 0.
\]

MDCR computes the threshold $s^*_{\text{MDCR}}$ as the empirical quantile of proxy scores $\{\tilde{s}_i\}_{i=1}^n$, where $\tilde{s}_i$ are independently drawn from $F_i(\cdot)$.
Let $\sigma^2_{\text{MC}}(\mathcal{D}_{\text{cal}}) = \Var(s^*_{\text{MDCR}} \mid \mathcal{D}_{\text{cal}})$ denote the variance induced by the Monte Carlo sampling step conditioned on the data. Since the distributions $F_i$ are non-degenerate, we have $\sigma^2_{\text{MC}}(\mathcal{D}_{\text{cal}}) > 0$.

By the Law of Total Variance, the marginal variance of the MDCR estimator is:
\begin{equation}
\label{eq:total_var}
\Var(s^*_{\text{MDCR}}) = \Var(\E[s^*_{\text{MDCR}} \mid \mathcal{D}_{\text{cal}}]) + \E[\Var(s^*_{\text{MDCR}} \mid \mathcal{D}_{\text{cal}})].
\end{equation}

To compare this with $\Var(s^*_{\text{DCR}})$, we consider the asymptotic regime ($n \to \infty$).
By the consistency of sample quantiles (given the Glivenko--Cantelli convergence of the stratified empirical CDF), $s^*_{\text{MDCR}}$ is an asymptotically unbiased estimator of $s^*_{\text{DCR}}$ (the true quantile of the mixture).
Therefore, the first term in Eq.~\eqref{eq:total_var} converges to the variance of the target:
\[
\Var(\E[s^*_{\text{MDCR}} \mid \mathcal{D}_{\text{cal}}]) \to \Var(s^*_{\text{DCR}}) \quad \text{as } n \to \infty.
\]
The second term in Eq.~\eqref{eq:total_var} remains strictly positive due to the irreducible Monte Carlo noise.
Consequently, for the asymptotic variance (AsyVar), we strictly have:
\[
\text{AsyVar}(s^*_{\text{MDCR}}) = \text{AsyVar}(s^*_{\text{DCR}}) + \E[\sigma^2_{\text{MC}}] > \text{AsyVar}(s^*_{\text{DCR}}).
\]
This proves that DCR achieves strictly lower variance than MDCR asymptotically.
\end{proof}

\section{Computational Complexity Analysis}
\label{app:complexity}

In this section, we analyze the computational complexity of our proposed method DCR, its stochastic variant MDCR, and the baseline TCPR \citep{fermanian2025transductive}. Let $n$ be the size of the calibration set, $m$ be the size of the test set, and $N = n + m$ be the total number of items.

\subsection{TCPR (Baseline)}
\label{app:complexity_tcpr}

The computational bottleneck of TCPR lies in the construction of the rank envelopes (linear or quantile). To ensure the validity of these envelopes with high probability $1-\delta$, TCPR relies on Monte-Carlo simulations of the rank vector.

\begin{itemize}
    \item \textbf{Simulation Step:} The algorithm generates $K$ independent realizations of the rank vector. 
    In a naive implementation, obtaining the ranks of $n$ calibration items among $N=n+m$ total items would require sorting all $N$ random variables, leading to a complexity of $\mathcal{O}(K \cdot N \log N)$. 
    However, in our efficient implementation, we directly sample $n$ sorted integers from the range $[1, N]$ without explicitly sorting the unobserved $m$ test items. This optimization reduces the complexity of the simulation step to $\mathcal{O}(K \cdot n \log n)$. 
    This explains the empirical behavior observed in Figure \ref{fig:runtime_analysis} (Left), where the runtime of TCPR is largely independent of the test set size $m$.

    \item \textbf{Optimization Step:} To find the optimal envelope parameters (e.g., width $c$ or quantile $\gamma$), TCPR iterates through the $K$ simulated samples. This typically involves a search procedure over the parameter space, checking the coverage condition against $K$ samples. 
\end{itemize}

In practice, $K$ must be large (e.g., $K=10^5$ as suggested in \citet{fermanian2025transductive}) to control the Monte-Carlo error. Consequently, the total complexity is dominated by $\mathcal{O}(K \cdot n \log n)$. While this optimized implementation decouples the runtime from $m$, the large constant factor $K$ results in significant computational overhead compared to DCR, particularly when $n$ is large, as shown in Figure \ref{fig:runtime_analysis} (Right).

\subsection{DCR (Proposed)}
DCR is a deterministic algorithm that computes the exact mixture distribution of the non-conformity scores.
\begin{itemize}
    \item \textbf{Distribution Construction:} For each of the $n$ calibration items, we construct the PMF of the Negative Hypergeometric distribution. The support of this distribution has size $m+1$. Calculating the PMF/CDF values takes $O(m)$ operations per item. Total cost: $O(nm)$.
    \item \textbf{Mixture Aggregation:} We aggregate $n$ individual CDFs to form the mixture CDF $\hat{F}_{mix}$. This summation takes $O(nm)$.
    \item \textbf{Threshold Computation:} Finding the $(1-\alpha)$-quantile from the mixture CDF (which is a step function supported on the discrete score grid) requires scanning the sorted unique score values. In the worst case, this is bounded by $O(nm)$.
\end{itemize}
The total complexity of DCR is $O(nm)$. Notably, it does not depend on a large simulation constant $K$, making it significantly faster than TCPR for typical dataset sizes.

\subsection{MDCR (Stochastic Approximation)}
MDCR approximates DCR by sampling a single rank realization for each calibration item.
\begin{itemize}
    \item \textbf{Sampling:} We draw $n$ independent samples from the Negative Hypergeometric distribution. With efficient sampling algorithms, this takes $O(n)$.
    \item \textbf{Threshold Estimation:} We compute $n$ scores and sort them to find the empirical quantile. This takes $O(n \log n)$.
\end{itemize}
The total complexity is $O(n \log n)$ (excluding the $O(m)$ cost to predict for test items, which is shared by all methods). MDCR is the most computationally efficient method, suitable for massive datasets, though it incurs greater variance than DCR as discussed in Theorem~\ref{thm:variance}.

\section{Additional Experiments}
\label{app:additionalexp}

\begin{figure}[h]
    \centering
    \includegraphics[width=0.6\linewidth]{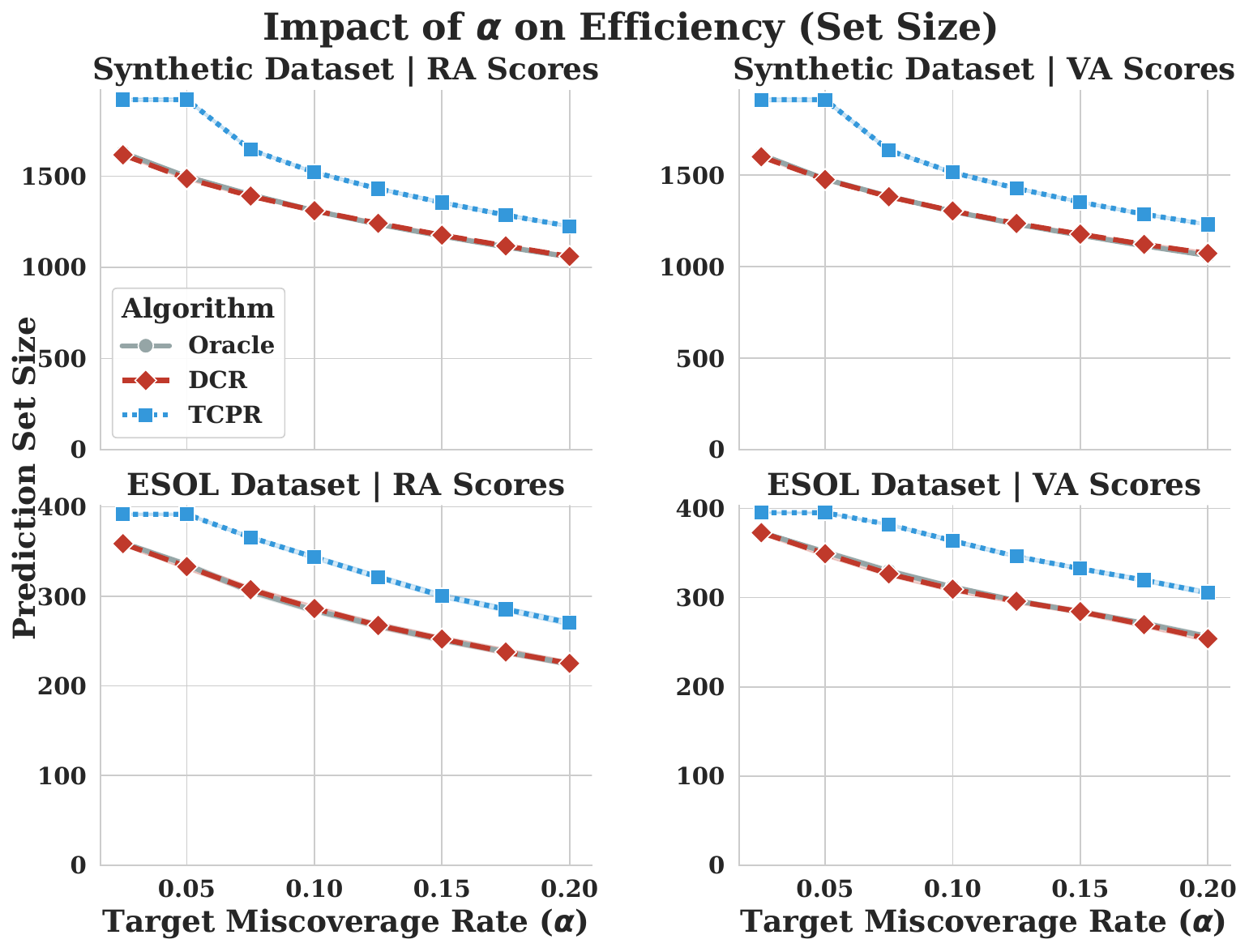}
    \caption{Impact of miscoverage rate $\alpha$ on prediction set size. The experiments are conducted on Synthetic (top) and ESOL (bottom) datasets using RA (left) and VA (right) scores.}
    \label{fig:ablation_alpha}
\end{figure}

\subsection{DCR achieves higher scalability across diverse reliability requirements than the TCPR baseline.}\label{add:alpha}
In Figure \ref{fig:ablation_alpha}, we analyze the prediction set size as a function of the target miscoverage rate $\alpha \in [0.025, 0.20]$ on the Synthetic ($n=500, m=1000$) and ESOL datasets. The results demonstrate that DCR maintains a near-linear relationship between the absolute prediction set size and $\alpha$, consistently delivering tighter and actionable sets. In contrast, TCPR tends to saturate at low $\alpha$ levels, where its set size effectively reaches a plateau, producing disproportionately large intervals that offer diminished utility. This contrast highlights that DCR ensures high utility even under strict reliability constraints, whereas TCPR struggles to maintain efficiency.

\begin{figure}[h]
    \centering
    \includegraphics[width=(\linewidth/3)*2]{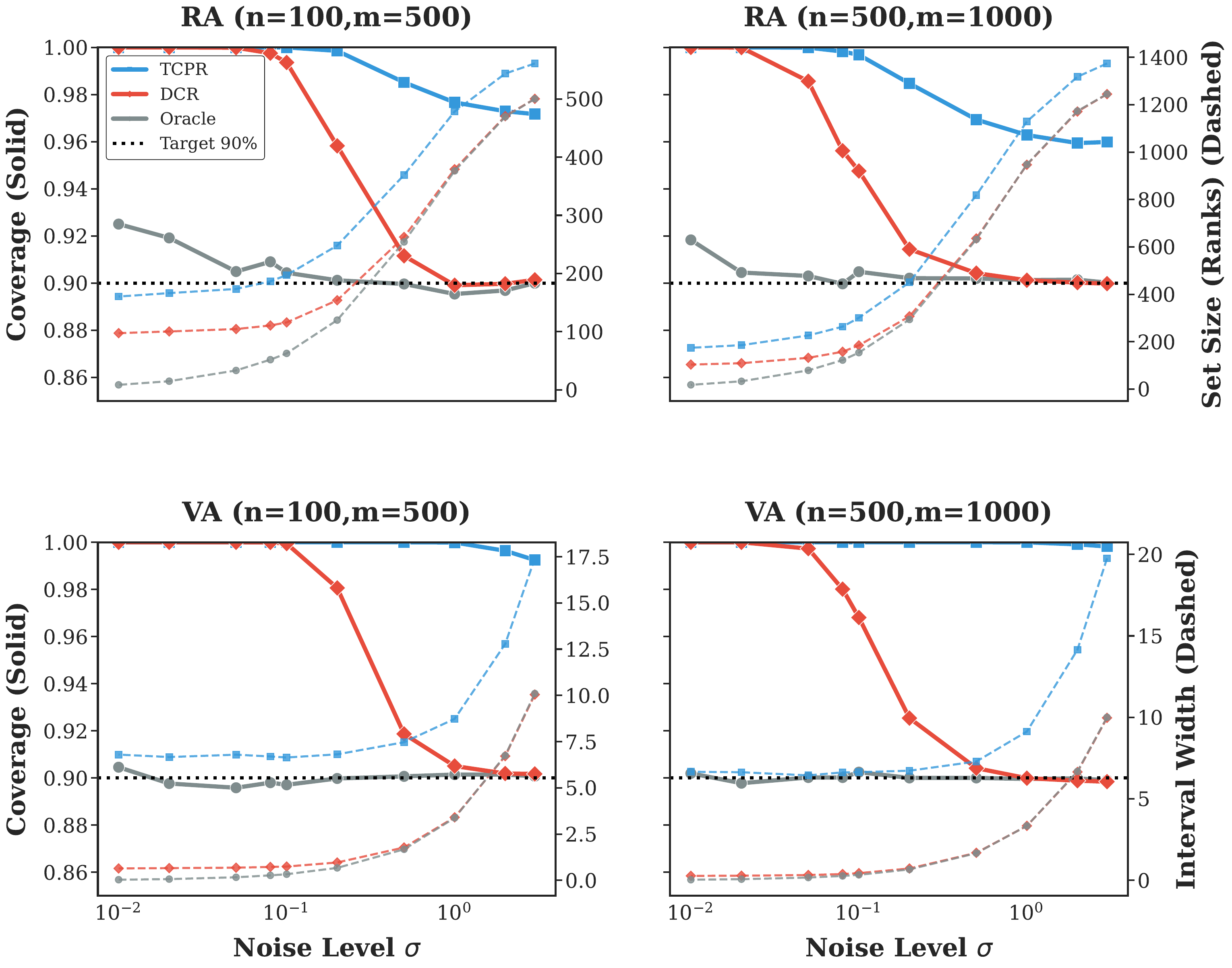}
    \caption{Impact of base model precision ($\sigma$) on uncertainty quantification.}
    \label{fig:model_quality}
\end{figure}
\subsection{DCR adapts to model precision, rapidly shedding conservatism as noise increases.}\label{add:presion}
Figure \ref{fig:model_quality} examines the impact of base model precision on uncertainty quantification, where a lower noise level $\sigma$ indicates higher precision. While TCPR exhibits persistent conservatism regardless of model quality, DCR demonstrates superior adaptivity. Consistent with Remark \ref{rem:conservatism}, DCR exhibits increased over-coverage in the high-precision regime ($\sigma < 0.1$) due to the dispersion mismatch between realized scores and the mixture distribution. However, this conservatism remains strictly lower than that of TCPR. As model precision decreases ($\sigma$ increases), this mismatch resolves, and DCR's coverage converges rapidly to the nominal target, whereas TCPR remains unnecessarily conservative across all regimes.

\begin{figure}[h]
    \centering
    \includegraphics[width=0.6\linewidth]{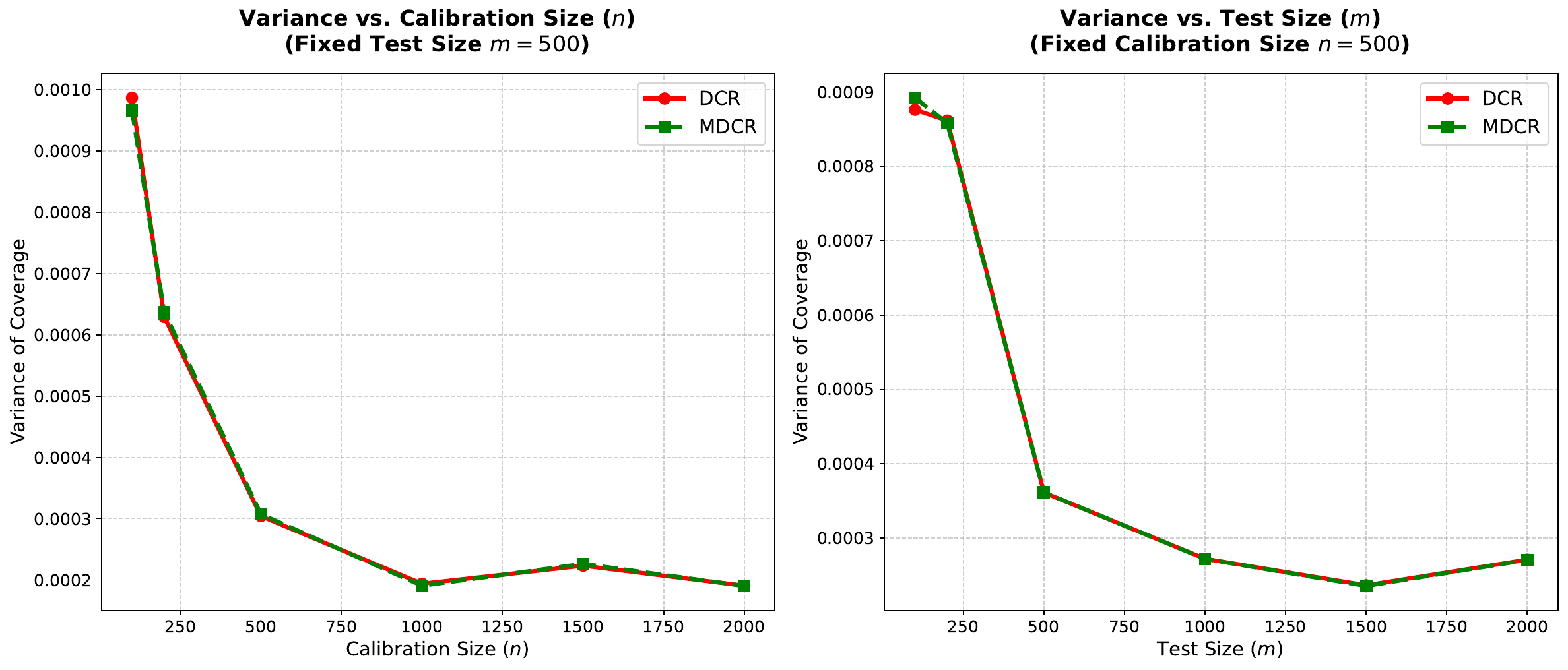}
    \caption{Comparison of the variance of realized coverage between DCR and MDCR on synthetic data. Left: Varying calibration size $n$ with fixed test size $m=500$. Right: Varying test size $m$ with fixed calibration size $n=500$. Both methods exhibit diminishing variance as sample sizes increase, with the theoretical gap between them being empirically negligible.}
    \label{fig:variance_comparison}
\end{figure}

\subsection{Empirical comparison of variance between DCR and MDCR.}\label{add:variance}
While Theorem \ref{thm:variance} establishes that DCR achieves strictly lower asymptotic variance than MDCR, we investigate the practical magnitude of this difference on synthetic data. We conducted experiments varying the calibration size $n \in [100, 2000]$ (with fixed $m=500$) and test size $m \in [100, 2000]$ (with fixed $n=500$). For each setting, we repeated the trial 1,000 times and calculated the variance of the realized coverage.

The results, shown in Figure \ref{fig:variance_comparison}, reveal two key insights. First, the variance curves of DCR and MDCR are nearly indistinguishable across all settings, suggesting that the additional stochasticity from MDCR's sampling step is negligible compared to the intrinsic data sampling variability. Second, we observe a consistent downward trend in variance for both methods as either $n$ or $m$ increases. This behavior is consistent with the law of large numbers and our finite-sample concentration analysis (Theorem \ref{thm:fcp_bound}), confirming that larger calibration and test sets lead to progressively more stable uncertainty quantification. Consequently, MDCR offers a favorable trade-off, providing substantial computational speedups with virtually no practical loss in stability compared to the exact DCR solution.

\end{document}